\documentclass[12pt, draftclsnofoot, onecolumn]{IEEEtran}
\IEEEoverridecommandlockouts
\usepackage{amsfonts}
\usepackage{dsfont}
\usepackage{amssymb}
\usepackage{graphicx}
\usepackage{subfigure}
\usepackage{enumerate}
\usepackage{amsmath}
\usepackage{color}
\usepackage{amsthm}
\usepackage{amsmath}
\usepackage{algorithm}
\usepackage{algpseudocode}
\usepackage[bookmarks=false]{hyperref}
\hypersetup{hidelinks}
\usepackage{breakurl}
\usepackage{cite}
\usepackage{epstopdf}

\usepackage{verbatim}

\allowdisplaybreaks[4]
\DeclareMathOperator*{\argmin}{arg\,min}

\newcommand{\bp}{\begin{proof} \small }
\newcommand{\ep}{\end{proof} \normalsize}
\newcommand{\epx}{\end{proof} \small}
\newcommand{\bpa}{\begin{proofappx} \footnotesize }
\newcommand{\epa}{\end{proofappx} \small }
\newtheorem{theorem}{Theorem}

\newtheorem{lemma}{Lemma}
\newtheorem{assumption}{Assumption}

\newtheorem*{theorem*}{Theorem}
\newtheorem*{proposition*}{Proposition}
\newtheorem*{corollary*}{Corollary}
\newtheorem*{lemma*}{Lemma}
\newtheorem*{assumption*}{Assumption}
\newtheorem*{definition*}{Definition}
\newtheorem*{claim*}{Claim}

\newcommand{\be}{\begin{equation}}
\newcommand{\ee}{\end{equation}}
\newcommand{\bs}{\begin{subequations}}
\newcommand{\es}{\end{subequations}}
\newcommand{\bq}{\begin{eqnarray}}
\newcommand{\eq}{\end{eqnarray}}
\newcommand{\bqn}{\begin{eqnarray*}}
\newcommand{\eqn}{\end{eqnarray*}}

\newcommand{\ba}{\left[ \begin{array}}
\newcommand{\ea}{\\ \end{array} \right]}
\newcommand{\ben}{\begin{enumerate}}
\newcommand{\een}{\end{enumerate}}

\def\g{{\boldsymbol{g}}}

\def\v{{\boldsymbol{v}}}
\def\w{{\boldsymbol{w}}}
\def\x{{\boldsymbol{x}}}
\def\y{{\boldsymbol{y}}}
\def\z{{\boldsymbol{z}}}

\def\real{{\mathchoice%
{\hbox{\rm\setbox1=\hbox{I}\copy1\kern-.45\wd1 R}}
{\hbox{\rm\setbox1=\hbox{I}\copy1\kern-.45\wd1 R}}
{\hbox{\scriptsize\rm\setbox1=\hbox{I}\copy1\kern-.45\wd1 R}}
{\hbox{\scriptsize\rm\setbox1=\hbox{I}\copy1\kern-.45\wd1 R}}}}

\def\Zint{{\mathchoice{\setbox1=\hbox{\sf Z}\copy1\kern-.75\wd1\box1}
{\setbox1=\hbox{\sf Z}\copy1\kern-.75\wd1\box1}
{\setbox1=\hbox{\scriptsize\sf Z}\copy1\kern-.75\wd1\box1}
{\setbox1=\hbox{\scriptsize\sf Z}\copy1\kern-.75\wd1\box1}}}
\newcommand{\complex}{ \hbox{\rm C\kern-0.45em\rule[.07em]{.02em}{.58em}%
\kern 0.43em}}

\begin{document}
	%
\title{Dynamic \!\!\! Scheduling \!\!\! for \!\!\! Over-the-Air \!\! Federated Edge Learning with Energy Constraints}

\author{ 
	Yuxuan Sun,~\IEEEmembership{Member,~IEEE,}
	Sheng Zhou,~\IEEEmembership{Member,~IEEE,} \\
	Zhisheng Niu,~\IEEEmembership{Fellow,~IEEE,} 
	Deniz G\"und\"uz,~\IEEEmembership{Senior Member,~IEEE}
	\thanks{ 
	Y. Sun, S. Zhou and Z. Niu are with the Beijing National Research Center for Information Science and Technology, Department of Electronic Engineering, Tsinghua University, Beijing 100084, China (e-mail: sunyuxuan@tsinghua.edu.cn, sheng.zhou@tsinghua.edu.cn, niuzhs@tsinghua.edu.cn).}  
	\thanks{D. G\"und\"uz is with the Department of Electrical and Electronic Engineering, Imperial College London, London SW7 2BT, UK (e-mail:  d.gunduz@imperial.ac.uk).}
	\thanks{Part of this work has been presented in IEEE ICC 2020 \cite{Sun2020ICC}. }
}
	
\maketitle

\begin{abstract}
	Machine learning and wireless communication technologies are jointly facilitating an intelligent edge, where federated edge learning (FEEL) is a promising training framework. As wireless devices involved in FEEL are resource limited in terms of communication bandwidth, computing power and battery capacity, it is important to carefully schedule them to optimize the training performance. In this work, we consider an over-the-air FEEL system with analog gradient aggregation, and propose an energy-aware dynamic device scheduling algorithm to optimize the training performance under energy constraints of devices, where both communication energy for gradient aggregation and computation energy for local training are included. 
	The consideration of computation energy makes dynamic scheduling challenging, as devices are scheduled before local training, but the communication energy for over-the-air aggregation depends on the $l_2$-norm of local gradient, which is known after local training. We thus incorporate estimation methods into scheduling to predict the gradient norm. Taking the estimation error into account, we characterize the performance gap between the proposed algorithm and its offline counterpart.
	Experimental results show that, under a highly unbalanced local data distribution, the proposed algorithm can increase the accuracy by $4.9\%$ on CIFAR-10 dataset compared with the myopic benchmark, while satisfying the energy constraints.
\end{abstract}

\begin{IEEEkeywords}
	Federated edge learning, over-the-air computation, energy constraints, dynamic scheduling, Lyapunov optimization.
\end{IEEEkeywords}
	
%
\IEEEpeerreviewmaketitle

\section{Introduction}
Many emerging applications at the wireless edge, such as autonomous driving, virtual reality and Internet of things (IoT), are powered by modern machine learning (ML) techniques. Data-driven approaches also penetrate into the wireless network itself for channel estimation, encoding and decoding, resource allocation, etc. \cite{MLintheair,Jiang2017ML}. 
The complex ML models for these applications need to be trained over massive datasets, while data samples are usually generated by edge devices. Traditional centralized training methods can hardly be competent, as collecting data at one location would create network congestion, lead to extremely high transmission cost and may cause privacy concerns. 
On the other hand, computing capabilities of base stations (BSs) and edge devices, such as mobile phones, smart vehicles and IoT sensors, are becoming increasingly powerful, enabling intensive computations at the edge \cite{Park2019wireless}.
In this context, federated learning (FL) is considered as a promising training framework that can exploit distributed data and computational resources with limited communication and privacy leakage \cite{googleai, flatscale, McMahan2017commun}. In FL, multiple devices train a shared model collaboratively with local data, and a central \emph{parameter server (PS)} coordinates training and aggregates global model periodically.

The limited communication resource and non-independent and identically distributed (i.i.d.) data, i.e., the distribution of local data at one device is not identical with that of other devices or the global data, are the two major challenges in FL \cite{li2019federated, zhao2018federated}. 
Current methods to improve the communication efficiency of FL mainly include model compression \cite{QSGD, FedPAQ, gradientspars, du2020high}, device scheduling \cite{yang2019scheduling, Amiri2021convergence}, and enabling multiple local iterations \cite{FedPAQ, wang2019adaptive}. Under non-i.i.d. data, the training performance can be improved by sharing global i.i.d. data with devices \cite{zhao2018federated} or the PS \cite{Yoshida2020HybridFL}, introducing data redundancy \cite{Sun2020ICC}, or scheduling devices based on their importance \cite{Amiri2021convergence}.

In a wireless network, FL can be carried out among wireless edge devices coordinated by a BS, called federated edge learning (FEEL). In FEEL, participating devices are often resource limited in terms of wireless bandwidth, computing capability and battery capacity. A key challenge is to design device scheduling and resource allocation algorithms that optimize the training performance under device energy constraints and training delay budget.
Considering the communication energy constraints, an energy-efficient bandwidth allocation policy is proposed to maximize the fraction of scheduled devices in \cite{zeng2019energy}, while an online algorithm is designed to maximize the sum utility of scheduling in \cite{Jie2020client}.
Due to the timeliness requirements of FEEL tasks at the wireless edge \cite{Sun2020edge}, training delay is also becoming a key performance metric.
The total communication delay for training is minimized by joint device selection and wireless resource allocation in \cite{Chen2020Convergence}, while the total training delay taking into account both local computations and model transmission is minimized in \cite{Shi2021Joint} by balancing the trade-off between the average delay per round and the total number of rounds required for convergence. Communication delay is combined with the importance of each update for probabilistic scheduling in \cite{Ren2020scheduling}.
A hierarchical FEEL framework is proposed in \cite{abad2019hierar}, where the end-to-end training delay is minimized by the joint optimization of update interval and model compression.
The trade-off between total energy consumption for communication and computation and the training delay is further considered in \cite{tran2019federated, yang2019energy, Mo2020energy}, yielding a joint design of local computation speed and wireless resource allocation.

The literature above mainly focuses on the implementation of FEEL via digital wireless communications.
However, the unique communication requirement of FEEL, i.e., the PS only needs the \emph{average} of local model updates rather than each individual vector, makes the separate design of learning and communication protocol highly suboptimal \cite{Gunduz2020communicate}. A new solution called \emph{over-the-air computation} is facilitated to further improve the communication efficiency \cite{Zhu2020toward, zhu2019broadband, mma2019federated, mma2020machine, Zhu2020tuning}, which is achieved by synchronizing the devices to transmit their local gradients or models in an analog fashion, and exploiting the superposition property of a wireless multiple access channel (MAC) to do the summation over-the-air. 
Power limits of devices can highly degrade the training performance, which yields the design of power allocation schemes over noisy channels \cite{mma2020machine}, fading channels \cite{mma2019federated} and broadband fading channels \cite{zhu2019broadband}. Power control algorithms that take into account the importance of updates \cite{Guo2021analog}, uplink and downlink noise \cite{Wei2021federated,mma2020convergence}, gradient statistics \cite{Zhang2020gradient} and non-i.i.d. data \cite{Sery2020het} are further proposed. While over-the-air FEEL requires accurate channel state information (CSI), it is shown in \cite{Amiri2020blind} that multiple antennas can help to relax the CSI requirement.
The impact of imperfect CSI or synchronization across devices is considered in \cite{Shao2021misaligned}, and a digital realization of over-the-air FEEL is further proposed in \cite{Zhu2020one}, based on one-bit gradient quantization and majority voting.

Existing papers on over-the-air FEEL mainly consider average power constraints for communication, but have not considered the computation energy for local model training, which is in fact non-negligible for edge devices.
In this work, we aim to optimize the training performance under total energy constraints of devices by designing an energy-aware dynamic device scheduling algorithm, where energy is consumed for both communication and computation.
The introduction of computation energy makes the scheduling decisions challenging due to the \emph{causality of decision making and energy consumption}. This is because, in over-the-air FEEL, the communication energy of each device for gradient aggregation depends on the $l_2$-norm of its local gradient estimate, which can only be obtained \emph{after computation}. However, online scheduling decision should be made at the start of each training round \emph{before computation}. Note that this issue does not arise in the case of digital communication, as the transmission power can be chosen independently of the local update.

The main contributions of this work are summarized as follows:

1) We characterize the convergence bound of the considered over-the-air FEEL system, based on which we formulate a device scheduling problem to optimize the training performance under the total energy budget of each device. Both the communication energy for gradient aggregation and the computation energy for local gradient calculation are included.

2) Due to the unavailability of future system states, we design an energy-aware dynamic device scheduling algorithm based on Lyapunov optimization, where a virtual queue is constructed to indicate the up-to-date energy deficit and enable online decision making.

3) To further address the challenge that communication energy is unknown at the device scheduling point, we propose estimation methods to predict the $l_2$-norm of the local gradients upon scheduling, and characterize the theoretical performance guarantee of the proposed scheduling algorithm by taking the error of energy estimation into consideration.

4) Experiments on MNIST and CIFAR-10 datasets validate that the proposed dynamic device scheduling algorithm can achieve higher model accuracies compared with the myopic benchmark, while satisfying the energy limits. Under a highly-non-i.i.d. scenario, the accuracy can be increased by $4.9\%$.
The impact of design parameters on the training performance and energy consumption are also evaluated to provide guidelines for practical implementations. 

The rest of this paper is organized as follows.
In Section \ref{sys}, we introduce the system model and problem formulation.
In Section \ref{conv}, we carry out convergence analysis.
The energy-aware dynamic device scheduling algorithm is developed in Section \ref{algo} with its performance guarantee.
Experimental results are shown in Section \ref{sim}, and conclusions are given in Section \ref{con}.

\begin{figure*}[!t]
	\centering
	\includegraphics[width=0.55\linewidth]{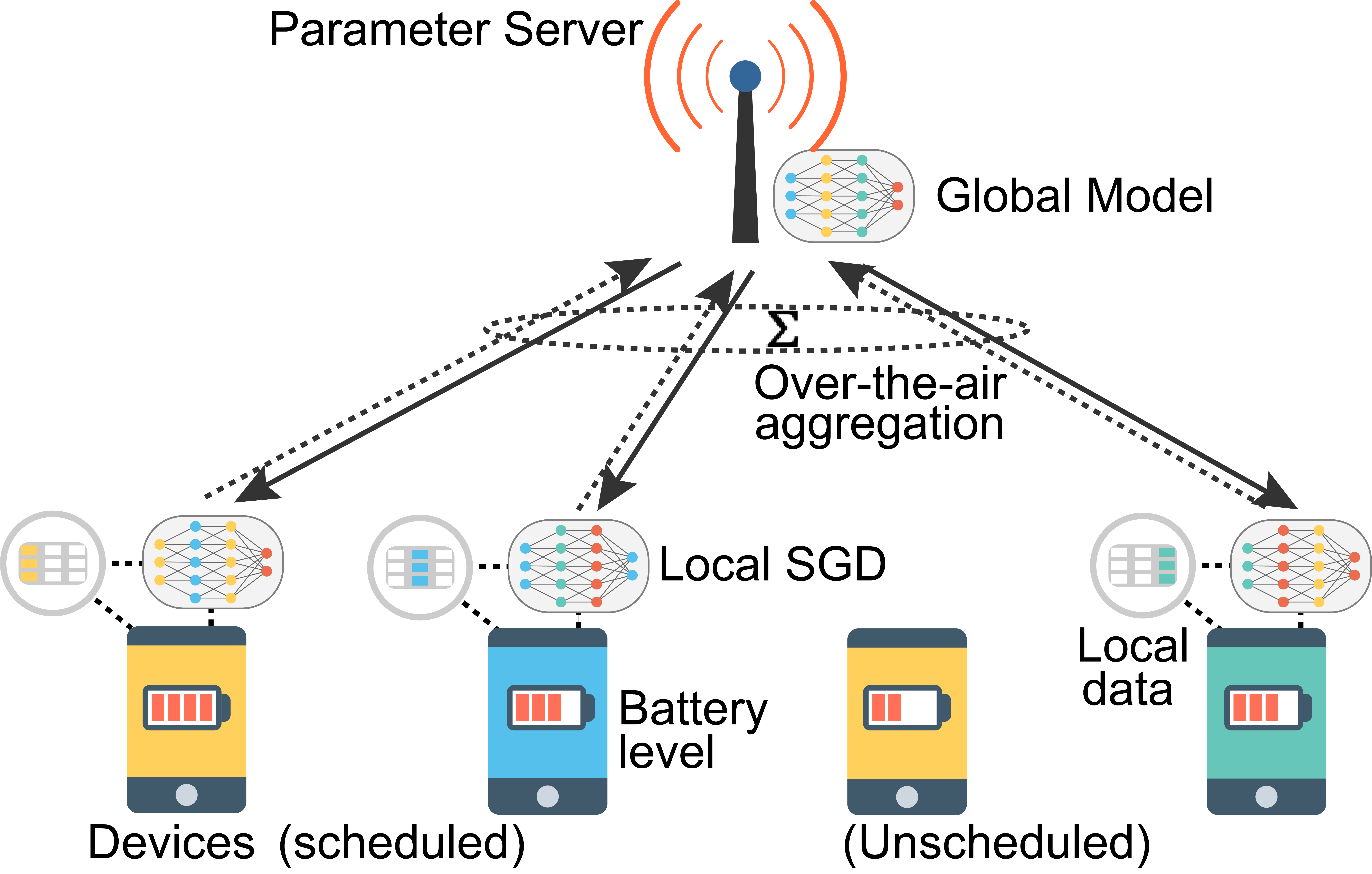}
	\vspace{-6mm}
	\caption{Illustration of the considered over-the-air FEEL system. }
	\label{system}
	\vspace{-6mm}
\end{figure*}

\section{System Model and Problem Formulation} \label{sys}

\subsection{System Overview}
As shown in Fig. \ref{system}, we consider a FEEL system with one PS and $N$ devices, denoted by $\mathcal{N}=\{1,2,\ldots,N\}$. Each device $n\in\mathcal{N}$ has a local dataset $\mathcal{D}_n$ with $D$ data samples, and the global dataset is denoted by $\mathcal{D}=\bigcup_{n=1,\ldots, N}\mathcal{D}_n$ with $ND$ data samples. 

Given a single data sample $\x\in\mathcal{D}$, a loss function $f(\w,\x)$ is used to measure the fitting performance of an $s$-dimensional model vector $\w\in \mathbb{R}^{s}$.
At device $n\in\mathcal{N}$, the local loss function $F_n(\w)$ is defined as the the average loss over local data samples, i.e., 
\begin{align}
F_n(\w)\triangleq \frac{1}{D}\sum_{\x \in\mathcal{D}_n} f(\w,\x).
\end{align}

The goal of the FEEL task is to train a shared global model $\w$ that minimizes the global loss function $F(\w)$, which is defined as
\begin{align}
F(\w)\triangleq\frac{1}{N}\sum_{n=1}^{N}F_n(\w)=\frac{1}{ND}\sum_{n=1}^{N}\sum_{\x \in\mathcal{D}_n} f(\w,\x).
\end{align}

Under the coordination of the PS, the FEEL system iterates the following three steps until the termination condition is satisfied: 1) the PS broadcasts the up-to-date global model to a subset of devices, which are scheduled to participate in the current training process; 2) the scheduled devices compute their local gradients with local datasets; and 3) the PS aggregates the local gradients over a wireless MAC and updates the global model. Each iteration consisting of these three steps is called a training round, which is indexed by $t$ in the following. The termination conditions that are commonly used for FEEL include the convergence of the global model, or reaching a preset maximum number of training rounds. Since we consider an energy-limited wireless scenario, we set the total number of training rounds to $T$.

\subsection{Local Gradient Computation}

At the start of the $t$-th training round, the PS schedules a subset of devices $\mathcal{B}_t\subseteq \mathcal{N}$, and broadcasts the global model vector $\w_{t-1}$ obtained in the last round to these scheduled devices. Let $\beta_{n,t}\in\{0,1\}$ be an indicator, with $\beta_{n,t}=1$ if device $n$ is scheduled to participate in the $t$-th training round, and $\beta_{n,t}=0$ otherwise. Thus $\mathcal{B}_t=\{n|\beta_{n,t}=1, n\in\mathcal{N}\}$. We also assume that the broadcast of $\w_{t-1}$ is error-free since the PS is a more capable node with sufficient power. 

Each scheduled device $n\in\mathcal{B}_t$ computes the local gradient estimate $\tilde{\g}_{n,t}$ by running the stochastic gradient descent (SGD) algorithm on a local mini-batch $\mathcal{L}_{n,t}\subseteq \mathcal{D}_n$, according to
\begin{align}\label{sgd_local}
\tilde{\g}_{n,t}=\frac{1}{L_b}\sum_{\x \in\mathcal{L}_{n,t}} \nabla f\left(\w_{t-1},\x\right),
\end{align}
where $L_b=|\mathcal{L}_{n,t}|\leq D$ is the batch size, and $\mathcal{L}_{n,t}$ is uniformly selected at random from the local dataset $\mathcal{D}_n$.
We remark here that, a single-iteration gradient update is considered in this work, but the proposed algorithm can be extended to a more general case where multiple local iterations are carried out in each training round.
Also note that, the batch size is considered as a hyper-parameter rather than an optimization variable, and set to an identical value across devices in this work, since local data might be non-i.i.d. across devices and local training should guarantee the fairness by exploiting the same amount of data.

We assume that for device $n$, the computation energy for calculating the local gradient on a single data sample is $e_n$, which can be estimated according to the number of floating point operations (FLOPs) of the ML model and the computation frequency of the device \cite{Mo2020energy}. Therefore, the computation energy consumption $E^{\text{[cp]}}_{n,t}$ at device $n$ in round $t$ is given by
\begin{align} \label{E_cp}
E^{\text{[cp]}}_{n,t}=e_nL_b .
\end{align}

\subsection{Gradient Aggregation Over-the-Air}
We assume that the devices transmit their local gradients over a noisy wireless MAC in an analog fashion for global gradient aggregation. To enable the summation of local gradients over-the-air, transmissions are synchronized across all the scheduled devices, and the transmit power of each device is aligned with the others. To be specific, let $h_{n,t}$ be the wireless channel gain between device $n$ and the PS, which is assumed to remain constant during one transmission period. Note that, the device scheduling policy designed in this work is applicable to arbitrary channel models. Moreover, as local gradient computation takes time and the wireless channel is time variant, the channel gain observed at the start of each training round may not be precise. The observation error, i.e., the difference between the observed channel gain that determines the device scheduling and its true value during transmission, will also be considered in the following. Let $\sigma_t$ be the power scalar that determines the received SNR at the PS. Then the transmit power $p_{n,t}$ of each scheduled device $n\in\mathcal{B}_t$ is set to 
\begin{align}
p_{n,t}=\frac{\sigma_t}{h_{n,t}},
\end{align} 
and $p_{n,t}\tilde{\g}_{n,t}$ is transmitted from device $n$ to the PS. The communication energy consumption $E^{\text{[tr]}}_{n,t}$ at device $n$ in round $t$ is then given by
\begin{align} \label{E_tr}
E^{\text{[tr]}}_{n,t}=\left\lVert  p_{n,t} \tilde{\g}_{n,t}\right\rVert_2^2
=\frac{\sigma_t^2 }{h_{n,t}^2}\left\lVert    \tilde{\g}_{n,t}\right\rVert_2^2,
\end{align}
where $\left\lVert \x\right\rVert_2$ represents the $l_2$-norm of vector $\x$.
Therefore, if device $n$ is scheduled in the $t$-th round, the total energy consumption $E_{n,t}$ for computation and communication is
\begin{align}
	E_{n,t}=E^{\text{[tr]}}_{n,t}+E^{\text{[cp]}}_{n,t}=\frac{\sigma_t^2 }{h_{n,t}^2}\left\lVert \tilde{\g}_{n,t}\right\rVert_2^2 +e_nL_b.
\end{align}

At the PS, the received signal $\y_t$ is given by
\begin{align} 
\y_t=\sum_{n \in\mathcal{B}_t} h_{n,t} p_{n,t} \tilde{\g}_{n,t}+\z_t 
=\sigma_t\sum_{n\in\mathcal{B}_t}\tilde{\g}_{n,t}+\z_t,
\end{align}
where $\z_t\in\mathbb{R}^s$ is an additive white Gaussian noise vector, in which each entry is i.i.d. and follows Gaussian distribution with zero mean and variance $\sigma_0^2$.

In FEEL, we aim to update the global model vector $\w_t$ according to 
\begin{align}
\w_t=\w_{t-1}-\eta_t \frac{\sum_{n\in\mathcal{B}_t }\tilde{\g}_{n,t}}{|\mathcal{B}_t|},
\end{align}
where $\eta_t$ is the learning rate in the $t$-th training round, and $|\cdot|$ denotes the cardinality of a set. Due to channel noise, the actual global model is updated according to
\begin{align} \label{global_model}
\w_{t}=\w_{t-1}-\eta_t\frac{\y_t}{\sigma_t |\mathcal{B}_t|}
=\w_{t-1}-\eta_t\left(\frac{\sum_{n\in\mathcal{B}_t}\tilde{\g}_{n,t}}{|\mathcal{B}_t|} +\frac{\z_t}{\sigma_t |\mathcal{B}_t|}\right).
\end{align}

\subsection{Problem Formulation} \label{formulation}
Given the total number of training rounds $T$ and the initial global model vector $\w_0$, we aim to minimize the expected global loss $\mathbb{E}[F(\w_T)]$ under the energy constraints of devices, by optimizing the device scheduling $\{\beta_{n,t}\}$ and power scalar $\{\sigma_t\}$. The expectation $\mathbb{E}[F(\w_T)]$ is taken over the randomness of channel noise and data sampling for local SGD. The problem is formulated as
\begin{subequations}
	\begin{align}
	\mathcal{P}1: ~\min_{\left\{\sigma_t, ~\beta_{n,t}\right\}_{t=1}^T} &~~\mathbb{E}[F(\w_T)] \\
	~\text{s.t.} ~~~~&~\sum_{t=1}^{T}\beta_{n,t}E_{n,t} \leq \bar{E}_n,  ~ \forall n,   \label{P1_energy}\\
	&~\sigma_t>0, ~\beta_{n,t}\in\{0,1\},  ~\forall n, t.  \label{P1_cons2}
	\end{align}
\end{subequations}
In the first constraint \eqref{P1_energy}, $\bar{E}_n$ represents the total energy budget of device $n$, and the inequality indicates that for each device, the total energy consumption for both local gradient computation and wireless communication over $T$ training rounds cannot exceed its given budget. The second constraint limits the ranges of optimization variables. 

Based on the law of telescoping sums, problem $\mathcal{P}1$ can be re-written as
	\begin{align}
	\mathcal{P}2: \min_{\left\{\sigma_t, ~\beta_{n,t}\right\}_{t=1}^T} ~~&\sum_{t=1}^{T} ~\mathbb{E}[F(\w_t)]- \mathbb{E}[F(\w_{t-1})]\\
	~\text{s.t.} ~~~~~~~ &\text{constraints}~ \eqref{P1_energy}, ~\eqref{P1_cons2}. \nonumber
	\end{align}

There are three major challenges to solve problem $\mathcal{P}2$:

 \emph{1) The inexplicit form of the objective function:}
	Since the neural network architectures for ML might be deep and diverse, and the evolution of the model vector is very complex during the training process, it is hard to express $\mathbb{E}[F(\w_T)]$  or $\mathbb{E}[F(\w_t)]- \mathbb{E}[F(\w_{t-1})]$ in closed form.
		
\emph{2) The unavailability of future information:}
	Optimally solving problem $\mathcal{P}2$ requires the system states of future training rounds due to the existence of total energy constraints, which are not available in practice. Therefore, we aim to design an online scheduling algorithm in this work, which only relies on the system states in the current training round.
	
 \emph{3) The causality of decision making and energy consumption:}
	A unique characteristic of over-the-air FEEL is that the communication energy \eqref{E_tr} depends on the $l_2$-norm of local gradient through $\left\lVert    \tilde{\g}_{n,t}\right\rVert_2^2$, which can only be acquired after computing the gradient in each round. 
	However, online device scheduling decision should be made before gradient computation, in order not to consume computation energy at unscheduled devices, or even not to transmit global updates to these devices. This means that the exact energy consumption in the current training round is unknown upon decision making.
	Moreover, the channel gain $h_{n,t}$ observed at the start of each round may not be precise, making the estimation of communication energy more inaccurate.

To address these challenges, we first substitute the objective function with its upper bound based on the convergence analysis in Section \ref{conv}. Then in Section \ref{algo}, we design an online device scheduling algorithm based on Lyapunov optimization, where the unknown instantaneous states for decision making, including the $l_2$-norm of local gradients and the wireless channel gains, are substituted with their estimates, and in particular, the impact of the estimation error on the performance of the proposed algorithm is analyzed.

\section{Convergence Analysis and Problem Transformation} \label{conv}
In this section, we provide an upper bound for the objective function in problem $\mathcal{P}2$ based on the convergence analysis, and transform the original optimization problem to an alternative form with explicit expressions.

For the simplicity of notation, we define the local full gradient on device $n$ in the $t$-th round as $\g_{n,t}\triangleq\nabla F_n(\w_{t-1})= \frac{1}{D}\sum_{\x\in\mathcal{D}_n} \nabla f(\w_{t-1},\x) $, the global full gradient in round $t$ as $\g_t\triangleq\nabla F(\w_{t-1}) =   \frac{1}{N}\sum_{n=1}^{N} \g_{n,t}$, and the optimum loss as $F^*\triangleq\min_{\w\in\mathbb{R}^s} F(\w)$.

To facilitate the convergence analysis, we make the following assumptions according to the state-of-the-art literature, including \cite{QSGD,  FedPAQ, gradientspars, du2020high, mma2019federated, mma2020machine,zhu2019broadband, Zhu2020one}, etc.
\begin{assumption} \label{ass_unbias}
	Stochastic gradient is unbiased and variance-bounded, i.e., for any device $n$ and training round $t$, taking the expectation over stochastic data sampling, we have 
	\begin{align}
	&\mathbb{E}_{\x_n}\left[\nabla f\left(\w_{t-1},\x_n\right)\right]=\mathbb{E}_{\mathcal{L}_{n,t}}\left[\tilde{\g}_{n,t}\right ]=\g_t, 
	~~\mathbb{E}_{\x_n}\left[\left\lVert \nabla f\left(\w_{t-1},\x_n\right)-\g_{t}\right\rVert_2^2\right]\leq G^2, \forall n,t,
	\end{align}
	where $\x_n\in\mathcal{D}_n$ is a data sample, $\mathcal{L}_{n,t}\subseteq \mathcal{D}_n$ is a stochastic mini-batch, and $G$ is a constant.
\end{assumption}
\begin{assumption} \label{ass_smooth}
	Loss functions $F_1(\w),\ldots, F_N(\w)$ are $l$-smooth, i.e., for $\forall \v,\w \in \mathbb{R}^s$ and $n\in\mathcal{N}$,
	\begin{align}
	&F_n(\v)-F_n(\w)\leq \nabla F_n^{\text{T}}(\w)(\v-\w)+\frac{l}{2}\left\lVert \v-\w\right\rVert_2^2.
	\end{align}
\end{assumption}
\begin{assumption} \label{ass_convex}
	Loss functions $F_1(\w),\ldots, F_N(\w)$  are $\mu$-strongly convex, i.e., for $\forall \v,\w \in \mathbb{R}^s$ and $n\in\mathcal{N}$,
	\begin{align}
	F_n(\v)-F_n(\w)\geq \nabla F_n^{\text{T}}(\w)(\v-\w)+\frac{\mu}{2}\left\lVert \v-\w\right\rVert_2^2.
	\end{align}
\end{assumption}

\subsection{Convergence Analysis} \label{subsec_Tround_conv}
Based on the assumptions above, we provide a single-round convergence guarantee in the following lemma by characterizing the upper bound of $\mathbb{E}[F(\w_{t})]-\mathbb{E}[F(\w_{t-1})]$.

\begin{lemma} \label{single_round_fix}
	Given the global model vector $\w_{t-1}$ and the set of scheduled devices $\mathcal{B}_t$ at the beginning of round $t$, the single-round convergence is upper-bounded by
	\begin{align}
	\mathbb{E}[F(\w_{t})]-\mathbb{E}[F(\w_{t-1})]
	\leq  -\eta_t \left(1-\frac{l\eta_t}{2}\right) \lVert \g_t\rVert_2^2 +
	\frac{l\eta_t^2}{2}\left(\frac{G^2}{L_b |\mathcal{B}_t|}+\frac{\sigma_0^2 s}{\sigma_t^2|\mathcal{B}_t|^2}\right),
	\end{align}
	where the expectation is taken over the randomness of channel noise and SGD.
\end{lemma}
\begin{proof}
	See Appendix \ref{a1}.
\end{proof}

Based on Lemma \ref{single_round_fix}, the convergence performance of over-the-air FEEL after $T$ training rounds is given in the following theorem.
\begin{theorem} \label{T_rounds}
	Given the global model vector $\w_{0}$ and any device scheduling sequence $\{\mathcal{B}_t, t=1,\ldots, T\}$, after $T$ rounds of training,
	\begin{align}
	\mathbb{E}[F(\w_{T})]-F^* \leq (\mathbb{E}[F(\w_{0})]-F^*)\prod_{i=1}^{T} (1-\mu \eta_i )  +\sum_{i=1}^{T-1} A_i  \prod_{j=i+1}^{T}(1-\mu \eta_i ) +A_T,
	\end{align}
	where $A_t\triangleq\frac{\eta_t}{2}\left(\frac{G^2}{L_b |\mathcal{B}_t|}+\frac{\sigma_0^2 s}{\sigma_t^2|\mathcal{B}_t|^2}\right)$ and the learning rate satisfies $\eta_t\leq \min\{\frac{1}{l},1\}, ~\forall t$.
\end{theorem}
\begin{proof}
		See Appendix \ref{a2}.
\end{proof}

According to Lemma \ref{single_round_fix} and Theorem \ref{T_rounds}, we can see that the number of devices $ |\mathcal{B}_t|$ scheduled in each round makes a key contribution to the convergence rate of training. While existing papers \cite{zeng2019energy, Jie2020client, Sun2020ICC} maximize the weighted fraction of devices scheduled over time, we provide a more reasonable objective function based on the theoretical characterization.
We also remark that, although Assumption \ref{ass_unbias} indicates i.i.d. local data, our proposed algorithm can also work well under non-i.i.d. data as being validated in the experiments in Section \ref{sim}.
	
\subsection{Problem Transformation}
As discussed in Section \ref{formulation}, the objective function $\sum_{t=1}^{T} \mathbb{E}[F(\w_t)]- \mathbb{E}[F(\w_{t-1})]$ in problem $\mathcal{P}2$ cannot be expressed explicitly. To make the optimization problem tractable, we substitute the objective function with its convergence bound according to Lemma \ref{single_round_fix}, and formulate an alternative optimization problem:
	\begin{align}
	\mathcal{P}3: \min_{\left\{\sigma_t, ~\beta_{n,t}\right\}_{t=1}^{T}} ~&\sum_{t=1}^{T}   -\eta_t \left(1-\frac{l\eta_t}{2}\right) \lVert \g_t\rVert_2^2 +\frac{l\eta_t^2}{2}\left(\frac{G^2}{L_b\sum_{n=1}^{N}\beta_{n,t}}+\frac{\sigma_0^2 s}{\sigma_t^2\left(\sum_{n=1}^{N}\beta_{n,t}\right)^2}\right) \label{p3_obj}\\
	~\text{s.t.} ~~~~~~ & \text{constraints}~ \eqref{P1_energy}, ~\eqref{P1_cons2}, \nonumber
	\end{align}
where we recall that $\sigma_t$ and $\beta_{n,t}$ are the power scalar and worker scheduling indicator, respectively.

Moreover, due to the unavailability of future system states, we aim to design an online algorithm to solve problem $\mathcal{P}3$, and ignore the impact of current decision on the future system states. 
As the global full gradient $\g_t$ defined on the whole dataset is fixed given the global model vector $\w_{t-1}$ at the start of training round $t$, and the learning rate $\eta_t$ and smoothness parameter $l$ are hyper-parameters, the first term in \eqref{p3_obj} is a constant. Therefore, we ignore this term and transform the optimization problem to
\begin{align}
\mathcal{P}4: \min_{\left\{\sigma_t, ~\beta_{n,t}\right\}_{t=1}^{T}} ~&\sum_{t=1}^{T} \frac{l\eta_t^2}{2}\left(\frac{G^2}{L_b\sum_{n=1}^{N}\beta_{n,t}}+\frac{\sigma_0^2 s}{\sigma_t^2\left(\sum_{n=1}^{N}\beta_{n,t}\right)^2}\right) \label{p4_obj}\\
~\text{s.t.} ~~~~~~ & \text{constraints}~ \eqref{P1_energy}, ~\eqref{P1_cons2}. \nonumber
\end{align} 

\section{Energy-Aware Dynamic Device Scheduling Algorithm} \label{algo}

In this section, we propose an energy-aware dynamic device scheduling algorithm that solves problem $\mathcal{P}4$ in an online fashion. To address the challenge brought by the causality of decision making and communication energy consumption, we first propose two heuristics to estimate the $l_2$-norm of local gradient estimates. Then, we design an online scheduling algorithm based on Lyapunov optimization, and characterize the worst-case performance of the proposed algorithm, which takes the error of energy estimation into consideration. Finally, we provide some practical  considerations for real implementations.

\subsection{Estimating the $l_2$-Norm of Local Gradients} \label{subsec_esti}

We propose two heuristics in the following to estimate the $l_2$-norm of local gradients $\left \lVert \tilde{\g}_{n,t}\right\rVert_2^2$ at the start of each training round $t$.

\noindent\emph{1) Compute the $l_2$-norm of local gradients with a smaller mini-batch (EST-C):}

An additional step is introduced at the start of each training round. To be specific, the PS broadcasts the up-to-date global model $\w_{t-1}$ to all the devices. Then, each device randomly selects a mini-batch $\mathcal{L}'_{n,t}\subseteq\mathcal{D}_n$ with batch size $L_e$ to calculate a local gradient estimate $\tilde{\g}^\text{[est]}_{n,t}$:
\begin{align}
\tilde{\g}^\text{[est]}_{n,t}=\frac{1}{L_e}\sum_{\x \in\mathcal{L}'_{n,t}} \nabla f\left(\w_{t-1},\x\right) . \label{sgd_es}
\end{align}
The computation energy should be modified as $E^{\text{[cp]}}_{n,t}=\beta_{n,t}e_n(L_b-L_e)+ e_nL_e$.

We further assume that each device can upload the value of $\left\lVert \tilde{\g}^\text{[est]}_{n,t}\right\rVert_2^2$ to the PS with negligible cost, which is used as the estimation of the $l_2$-norm of local gradient for device $n$ in round $t$.

Let $G_{n,t}^2$ be the variance of the stochastic gradient on a single data sample for device $n$ in round $t$. Following the same proof of Lemma 1 as \eqref{a1_2}, the exact value of gradient norm and its estimation have the following expectations:
\begin{align}\label{est-c-dev}
&\mathbb{E}\left[\left\lVert \tilde{\g}_{n,t}\right\rVert_2^2 \right]=\left\lVert \g_{n,t}\right\rVert_2^2+\frac{G_{n,t}^2}{L_b}, 
~\mathbb{E}\left[\left\lVert \tilde{\g}^\text{[est]}_{n,t}\right\rVert_2^2 \right]=\left\lVert \g_{n,t}\right\rVert_2^2+\frac{G_{n,t}^2}{L_e}.
\end{align}
As $L_b$ is typically much larger than $L_e$, the expressions above indicate that the estimation may suffer from a large deviation due to the gradient variance.

\noindent\emph{2) Estimate with past information (EST-P):}

A simpler and more straightforward way is to use the most recent $l_2$-norm of local gradient to estimate the current one at each device. 
Let $t_n=\arg\max_t \{t|\beta_{n,t}=1\}$ be the most recent round in which device $n$ is scheduled. The estimated $l_2$-norm of the current local gradient estimate is: 
\begin{align}
	\left \lVert \tilde{\g}_{n,t}^\text{[est]}\right\rVert_2^2 & =	\left \lVert \tilde{\g}_{n,t_n}\right\rVert_2^2.
\end{align}

We will show in Fig. \ref{power_est_MNIST} and Fig. \ref{power_est_CIFAR} in Section \ref{sim} that, under our considered datasets, estimation by EST-P is more accurate due to the strong temporal correlation of gradients. Moreover, compared with EST-C that requires additional computation and communication, EST-P method is computation-free and only needs each device to report the $l_2$-norm of local gradients when it is scheduled. Therefore, we use EST-P to estimate the $l_2$-norm of local gradient in the following device scheduling algorithm.

\subsection{Energy-Aware Dynamic Device Scheduling Algorithm}

To enable online scheduling without any future information while satisfying the total energy constraints of devices, we construct a virtual queue $q_{n,t}$ for each device $n$ to indicate the gap between the cumulative energy consumption till round $t$ and the budget, evolved as
\begin{align} \label{queue}
q_{n,t+1}=\max\left\{ q_{n,t}+ \beta_{n,t}E_{n,t}-\frac{\bar{E}_n}{T},0 \right\},
\end{align}
with initial value $q_{n,1}=0$, $\forall n\in\mathcal{N}$. 

Recall that the causality of device scheduling and energy consumption leads to the unawareness of $E_{n,t}$ at the start of round $t$. Based on the estimated $l_2$-norm of local gradient $\left\lVert \tilde{\g}_{n,t}^{\text{[est]}}\right\rVert_2^2$ by EST-P and the wireless channel gain $\tilde{h}_{n,t}$ observed at the beginning of round $t$, the estimated energy consumption of device $n$ at round $t$, denoted by $\tilde{E}_{n,t}$, is given by
\begin{align}\label{Ent_es}
	\tilde{E}_{n,t}=\frac{\sigma_t^2 }{\tilde{h}_{n,t}^2}\left\lVert \tilde{\g}_{n,t}^{\text{[est]}}\right\rVert_2^2 +e_nL_b. 
\end{align}

Let $U_t\triangleq \frac{l\eta_t^2}{2}\left(\frac{G^2}{L_b\sum_{n=1}^{N}\beta_{n,t}}+ \frac{\sigma_0^2s}{\sigma_t^2\left(\sum_{n=1}^{N}\beta_{n,t}\right)^2}\right)$. Inspired by the drift-plus-penalty algorithm of Lyapunov optimization \cite{neely2010stochastic}, the online scheduling aims to solve the following problem:
\begin{subequations}
	\begin{align}\label{prob_online}
	\mathcal{P}5: ~\min_{\left\{\sigma_t,~\beta_{n,t}\right\}} &VU_t + \sum_{n=1}^{N}\beta_{n,t}q_{n,t}\tilde{E}_{n,t} \\
	~\text{s.t.} ~~~&\sigma_t>0, ~\beta_{n,t}\in\{0,1\},  ~\forall n, 
	\end{align}
\end{subequations}
where $V$ is an adjustable weight parameter to balance the loss $U_t$ and energy consumption.
Compared to the classical drift-plus-penalty algorithm where all the states in the current round are known, the drift term  $q_{n,t}\tilde{E}_{n,t}$ in $ \mathcal{P}5$ is an approximation, and thus we call it \emph{estimated-drift-plus-penalty} algorithm.

Notice that problem $\mathcal{P}5$ is a mixed integer non-linear programming problem, which is still very difficult to solve. 
Meanwhile, existing work has shown that the convergence performance of FEEL with over-the-air gradient aggragation is not very sensitive to the power scalar $\sigma_t$, as long as the received SNR or the power limit of each device is larger than a threshold \cite{mma2020machine}. Therefore, we further decouple the optimization variables in $\mathcal{P}5$ by considering the power scalar $\sigma_t$ as a hyper-parameter, and then develop the optimal solution to the online device scheduling problem.

\noindent\emph{1) Received SNR and Power Scalar $\sigma_t$}

The power scalar $\sigma_t$ is chosen as follows.
In the $t$-th round, the expected received SNR at the PS side is denoted by $\gamma_t$, given by
\begin{align}
\gamma_t=\mathbb{E}\left[\frac{\left\lVert\sigma_t\sum_{n\in\mathcal{B}_t}\tilde{\g}_{n,t}\right\rVert_2^2}{\lVert \z_t \rVert_2^2 } \right]
=\frac{\sigma_t^2}{\sigma_0^2s}\left\lVert\sum_{n\in\mathcal{B}_t}\tilde{\g}_{n,t}\right\rVert_2^2.
\end{align}

Let $\gamma_0$ be a pre-defined SNR threshold. The power scalar is set according to 
\begin{align} \label{sigma_t}
	\sigma_t =~\frac{\gamma_0\sigma_0^2\sqrt{s}}{\min_{n\in\mathcal{N}}\left \lVert \tilde{\g}_{n,t}\right\rVert_2 }
	\approx ~\frac{\gamma_0\sigma_0^2\sqrt{s}}{\min_{n\in\mathcal{N}}\left \lVert \tilde{\g}_{n,t}^\text{[est]}\right\rVert_2 },
\end{align}
such that the expectation of the received SNR can meet the threshold $\gamma_0$ even in the worst case when a single device is scheduled. 
Recall that $\left \lVert \tilde{\g}_{n,t}\right\rVert_2 $ is unknown and thus approximated by $\left \lVert \tilde{\g}_{n,t}^\text{[est]}\right\rVert_2$ according to the EST-P method.

\noindent\emph{2) Optimal Online Device Scheduling}

Given the power scalar $\sigma_t$, the device scheduling $\{\beta_{n,t}\}$ in the $t$-th round aims to solve
\begin{subequations}
	\begin{align}\label{prob_worker}
	\mathcal{P}6: ~\min_{\left\{\beta_{n,t}\right\}} ~&VU_t + \sum_{n=1}^{N}\beta_{n,t}q_{n,t}\tilde{E}_{n,t} \\
	~\text{s.t.} ~~~&\beta_{n,t}\in\{0,1\},  ~\forall n.
	\end{align}
\end{subequations}

An optimal solution to problem $\mathcal{P}6$ is shown in Algorithm \ref{Algo_P6}.
In Line 1, we sort $\mathcal{C}_t=\left\{q_{n,t}\tilde{E}_{n,t}, \forall n\right\}$ in the ascending order, and let $C_t^{[m]}$ be the $m$-th smallest value of $\mathcal{C}_t$. 
Many sorting algorithms such as Heapsort or Mergesort can be used, with worst-case complexity $O(N\log N)$.
In Lines 2-4, we iterate over the possible number of scheduled devices $k\in\{1,\ldots,N\}$, and calculate the corresponding minimum estimated-drift-plus-penalty $v_t(k)$ according to
\begin{align} \label{v_k}
	v_t(k)\triangleq \frac{l\eta_t^2}{2}\left(\frac{G^2}{L_bk}+ \frac{\sigma_0^2s}{\sigma_t^2k^2}\right) + \sum_{n=1}^{k}C_t^{[k]}.
\end{align}
The optimal number of devices $k^*$ to be scheduled is obtained by finding the minimum $v_t(k)$ according to Line 5, and $k^*$ devices with smallest estimated drift $q_{n,t}\tilde{E}_{n,t}$ are scheduled, as shown Lines 6-8. Besides Line 1, all the other steps are with complexity $O(N)$, and thus the complexity of Algorithm \ref{Algo_P6} is $O(N\log N)$.

\begin{algorithm} [!t] 
	\caption{Optimal Online Device Scheduling to $\mathcal{P}6$}  \label{Algo_P6}
	\begin{algorithmic}[1] 
		\State Sort $\mathcal{C}_t=\left\{q_{n,t}\tilde{E}_{n,t}, \forall n\right\}$ and let $C_t^{[m]}$ be the $m$-th smallest value of $\mathcal{C}_t$.
		\For {$k=1,\ldots,N$}
		\State Calculate $v_t(k)$ according to \eqref{v_k}.
		\EndFor
		\State Get $k^*=\argmin_k \{v_t(k) \mid k=1,\ldots,N\}$.
		\For {$n=1,\ldots,N$}
		\State Let $\beta_{n,t}=1$ if $q_{n,t}\tilde{E}_{n,t}\leq C_t^{[k^*]}$, and $\beta_{n,t}=0$ otherwise.
		\EndFor
	\end{algorithmic}
\end{algorithm}

\begin{algorithm} [t] 
	\caption{Energy-Aware Dynamic Device Scheduling Algorithm}  \label{Algo_dyn}
	\begin{algorithmic}[1] 
		\State \textbf{Initialization}: initialize global model $\w_0$. Each device $n$ runs local SGD according to \eqref{sgd_local} to report $\left \lVert \tilde{\g}_{n,0}\right\rVert_2^2$ to the PS, and let $q_{n,1}=0$.
		\For {$t=1,\ldots, T$}
		\State The PS set $\sigma_t$ according to \eqref{sigma_t}, acquires channel gains $\tilde{h}_{n,t}$ and calculates the estimated energy consumption $\tilde{E}_{n,t}$ according to \eqref{Ent_es} for all devices. 
		\State The PS schedules a subset of devices $\mathcal{B}_t$ by solving $\mathcal{P}6$ according to Algorithm \ref{Algo_P6}.
		\State The PS broadcasts $\w_{t-1}$ and $\sigma_t$ to the scheduled devices $n\in\mathcal{B}_t$. 
		\State Each scheduled device $n\in\mathcal{B}_t$ updates local gradient $\tilde{\g}_{n,t}$ according to \eqref{sgd_local}, and transmits $\frac{\sigma_t }{h_{n,t}}\tilde{\g}_{n,t}$ simultaneously with all the other scheduled devices.
		\State The PS receives $\y_t$ and updates the global model $\w_t$ according to \eqref{global_model}.
		\State Each scheduled device $n\in\mathcal{B}_t$ reports $E_{n,t}$ and the PS updates the virtual queue $q_{n,t}$ for all devices according to \eqref{queue}.
		\EndFor
	\end{algorithmic}
\end{algorithm}

\noindent\emph{3) The Complete Algorithm}	

The proposed energy-aware dynamic device scheduling algorithm is summarized in Algorithm \ref{Algo_dyn}.
In the $t$-th training round, the PS makes device scheduling decision by solving $\mathcal{P}6$ based on the estimated energy consumption and the virtual queue, which is run in an online fashion without any future information.
The weight parameter $V$ and the virtual queue states $\{q_{n,t}, \forall n\}$ jointly balance the training gain of the FEEL task and the energy consumption of devices. In particular, a larger $V$ puts more emphasis on scheduling more devices so as to accelerate the convergence rate. Meanwhile, a larger $q_{n,t}$ indicates that the cumulative energy consumption of device $n$ till the current round far exceeds the budget, so that the device tends to save energy. As shown in Algorithm \ref{Algo_P6}, the optimal solution to $\mathcal{P}6$ also indicates that devices with smaller values of $q_{n,t}\tilde{E}_{n,t}$ are always scheduled first, as their energy is relatively sufficient.

Then, the up-to-date global model vector  $\w_{t-1}$ is broadcast to the scheduled devices, who run local SGD to compute local gradients $\tilde{\g}_{n,t}$ in parallel.
After computation, local gradients are aggregated over-the-air and the global model is updated by the PS. Finally, the PS collects the actual energy consumption of each scheduled device, which also contains the information of local gradient norm $\left \lVert \tilde{\g}_{n,t}\right\rVert_2$, and updates the virtual queue states for all the devices to guide the scheduling decision in the next training round.

\subsection{Performance Analysis} \label{Lyapunov_analysis}	

The performance of the proposed dynamic device scheduling algorithm is characterized by comparing with its optimal offline counterpart, which is achieved by solving the optimal device scheduling $\{\beta_{n,t}^*\}$ to problem $\mathcal{P}4$ while taking $\{\sigma_t\}$ as a pre-defined hyper-parameter sequence. 
Let $\sum_{t=1}^{T}  U_t^*$ be the offline optimal cumulative loss of $\mathcal{P}4$ by scheduling device sequence $\{\beta_{n,t}^*\}$, and define $\sum_{t=1}^{T}  U_t^\ddagger$ as the cumulative loss of the proposed algorithm, which is achieved by solving the online device scheduling problem $\mathcal{P}6$ in each round. To enable the theoretical analysis, we neglect the impact of current scheduling decision on the future gradient norm for the offline counterpart. 
Meanwhile, we do not limit the distributions of wireless channel or local gradient norms, which can be non-stationary over time.

The performance guarantee of the proposed algorithm is shown in the following theorem.

\begin{theorem} \label{theorem_algo}
	Compared to the offline optimal solution, the cumulative loss of Algorithm \ref{Algo_dyn} can be bounded by
	\begin{align} \label{bound_Ut}
		\sum_{t=1}^{T}  U_t^\ddagger \leq  \sum_{t=1}^{T}  U_t^*+\frac{\theta_0T^2+T(T-1)\delta_0\sum_{n=1}^{N} \theta_n}{V},
	\end{align}
	and the total energy consumption of Algorithm \ref{Algo_dyn} can be bounded by 
	\begin{align}\label{bound_energy}
		\sum_{t=1}^{T}\beta_{n,t}E_{n,t} \leq \bar{E}_n +\sqrt{2V\sum_{t=1}^{T}  U_t^*+2\theta_0T^2+2T(T-1)\delta_0\sum_{n=1}^{N} \theta_n}~,
	\end{align}
	where $\delta_0\triangleq\max_{\{n,t\}}\left\{\left|\tilde{E}_{n,t}-E_{n,t}\right|\right\}$, $\theta_0\triangleq\sum_{n=1}^{N}\frac{1}{2}\theta_n^2$ and $\theta_n\triangleq\max_t \left\{\left| E_{n,t}-\frac{\bar{E}_n}{T}\right| \right\}$.
\end{theorem}
\begin{proof}
	See Appendix \ref{a3}.
\end{proof}

Theorem \ref{theorem_algo} shows that, the training performance of the proposed energy-aware dynamic device scheduling algorithm can be bounded with respect to its optimal offline counterpart, while the deviation between the cumulative energy consumption of each device and its budget is also bounded. The worst-case performance can be improved by reducing the upper bound of the energy overuse $\theta_n$ and the maximum energy estimation error $\delta_0$. Moreover, the trade-off between the training performance of the FEEL task and maximum energy consumption of each device can be balanced by the weight parameter $V$.

We also remark here that, compared to the state-of-the-art that also applies Lyapunov optimization to solve scheduling problem under energy constraints \cite{Sun2020ICC, Jie2020client, sun2017emm}, our analysis further shows the impact of estimation error on the performance bound.

\subsection{Implementation Issues}

To enable the efficient implementation of the proposed algorithm in s real system, we provide some practical considerations as follows.

\noindent\emph{1) Communication Rescheduling}

The key motivation of rescheduling is to avoid using significantly more energy than expected when the estimation error of $\tilde{E}_{n,t}$ is large.
To be specific, after local gradient computation, each scheduled device can learn its exact energy consumption $E_{n,t}$ by calculating the local gradient norm $\left \lVert \tilde{\g}_{n,t}\right\rVert_2^2$ and acquiring the accurate channel gain $h_{n,t}$. 
If $E_{n,t}-\tilde{E}_{n,t}\leq\delta_h $, where $\delta_h>0$ is a given threshold, then the device is scheduled for gradient aggregation. Otherwise, the device backs off from the communication step.

\noindent\emph{2) Minimum Value of Virtual Queue}

The typical evolution of virtual queue is given in \eqref{queue}, in which the minimum queue value is set to $0$. In problem $\mathcal{P}6$, $q_{n,t}=0$ indicates that the energy consumption is not considered in the current scheduling round, and thus the device is scheduled. However, the energy consumption $E_{n,t}$ might be large, leading to a large deviation $\theta_n$ and thus a poor worst-case performance.
To avoid such cases, we instead set $q_{\text{min}}>0$ as the minimum value of the virtual queue in practice. 

\noindent\emph{3) Estimations of Smoothness Parameter $l$ and Variance Bound $G^2$}

Our algorithm is designed based on the convergence analysis under Assumptions in Section \ref{conv}. These hyper-parameters should be estimated in practice. According to the definition of smoothness, $l$ is estimated by the maximum value of $\frac{\left \lVert \tilde{\g}_{n,t}-\tilde{\g}_{n,t-1}\right\rVert}{\left \lVert\w_{t-1}-\w_{t-2}\right\rVert}$ during training, while each device can count the variance of local gradients to set a reasonable variance bound $G^2$.

\section{Experiments} \label{sim}

In this section, we evaluate the proposed energy-aware dynamic device scheduling algorithm for an image classification task using both MNIST\footnote{http://yann.lecun.com/exdb/mnist/} and CIFAR-10\footnote{https://www.cs.toronto.edu/~kriz/cifar.html} datasets.
We consider $N=10$ devices and both i.i.d. and non-i.i.d. datasets on devices. 
For the i.i.d. case, the training dataset of MNIST with 60000 samples (or CIFAR-10 with 50000 samples) is randomly partitioned into $N$ disjoint subsets, and each device holds one subset.
For the non-i.i.d. case, we sort the data samples by their labels, and each device holds a disjoint subset of data with $m$ labels (represented by `non-i.i.d. ($m$)' in the following). Note that the data distributions are more skewed for smaller $m$, and they become i.i.d. when $m$ is equal to the total number of classes in the dataset.

For MNIST, we train a multilayer perceptron (MLP) which has a 784-unit input layer with ReLU activation, a 64-unit hidden layer, and a 10-unit softmax output layer, with 50890 parameters in total. 
The total number of rounds is set to $T=200$, and 10 local iterations are carried out per round with batch size $L_b=64$. 
In each round, the total computation energy is $1\mathrm{J}$ for each device.
For CIFAR-10, we train a convolutional neural network (CNN) with the following structure:
two $3\times 3$ convolution layers each with $32$ channels and followed by a $2\times 2$ max pooling layer, two $3\times 3$ convolution layers each with $64$ channels and followed by a $2\times 2$ max pooling layer, a fully connected layer with 120 units, and finally a 10-unit softmax output layer. Each convolution or fully connected layer is activated by ReLU, and the total number of model parameters is 258898. We train the model for $T=10000$ rounds, and one mini-batch is run per round with batch size $L_b=64$. Local computation energy per round per device is set to $10\mathrm{J}$.

For both MNIST and CIFAR-10, the learning rate $\eta_t$ is set to $0.05$, $\forall t$, a momentum of $0.9$ is adopted, and cross entropy is adopted as the loss function. The wireless channel follows Rayleigh fading with scale parameter 1, and by default we assume that the accurate channel gain can be observed, i.e., $\tilde{h}_{n,t}=h_{n,t}$. The variance of channel noise is $\sigma_0^2=10^{-6}$. The power scalar is selected according to \eqref{sigma_t}, where the default SNR threshold is $\gamma_0=5$. For the dynamic scheduling algorithm, the minimum value of virtual queue is $q_{\text{min}}=0.1$, and the maximum estimation error $\delta_h=0.5\tilde{E}_{n,t}$ is allowed for communication reschedule.

\subsection{$l_2$-Norm of Local Gradients}
\begin{figure*}[!t]
	\centering
	\vspace{-5mm}	
	\subfigure[I.i.d. local data.]{\label{power_est_MNIST_iid}			
		\includegraphics[width=0.4\textwidth]{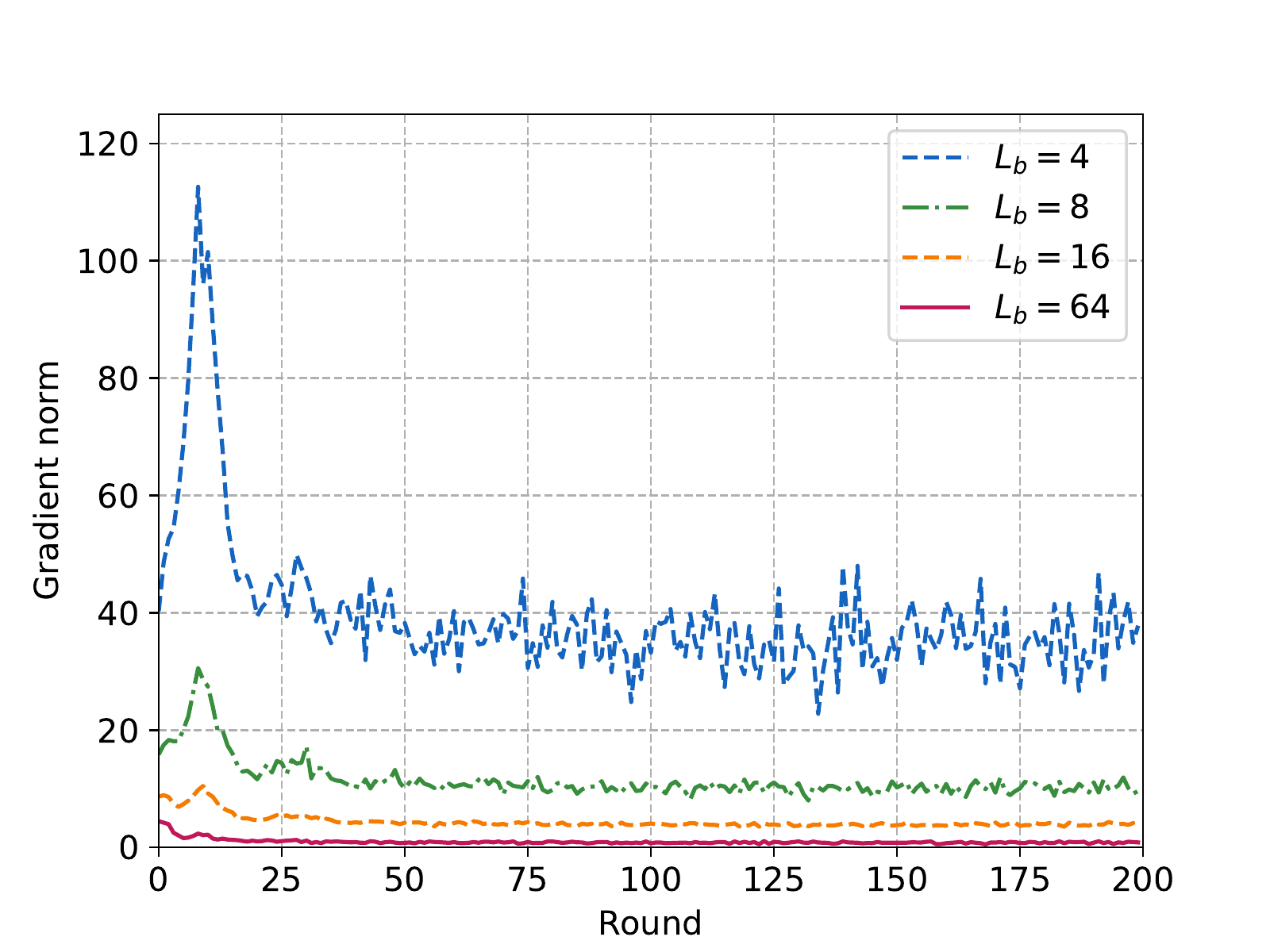} 	\vspace{-1mm}}		
	\hspace{8mm}
	\subfigure[Non-i.i.d. ($m=1$).]{\label{power_est_MNIST_label1}	
		\includegraphics[width=0.4\textwidth]{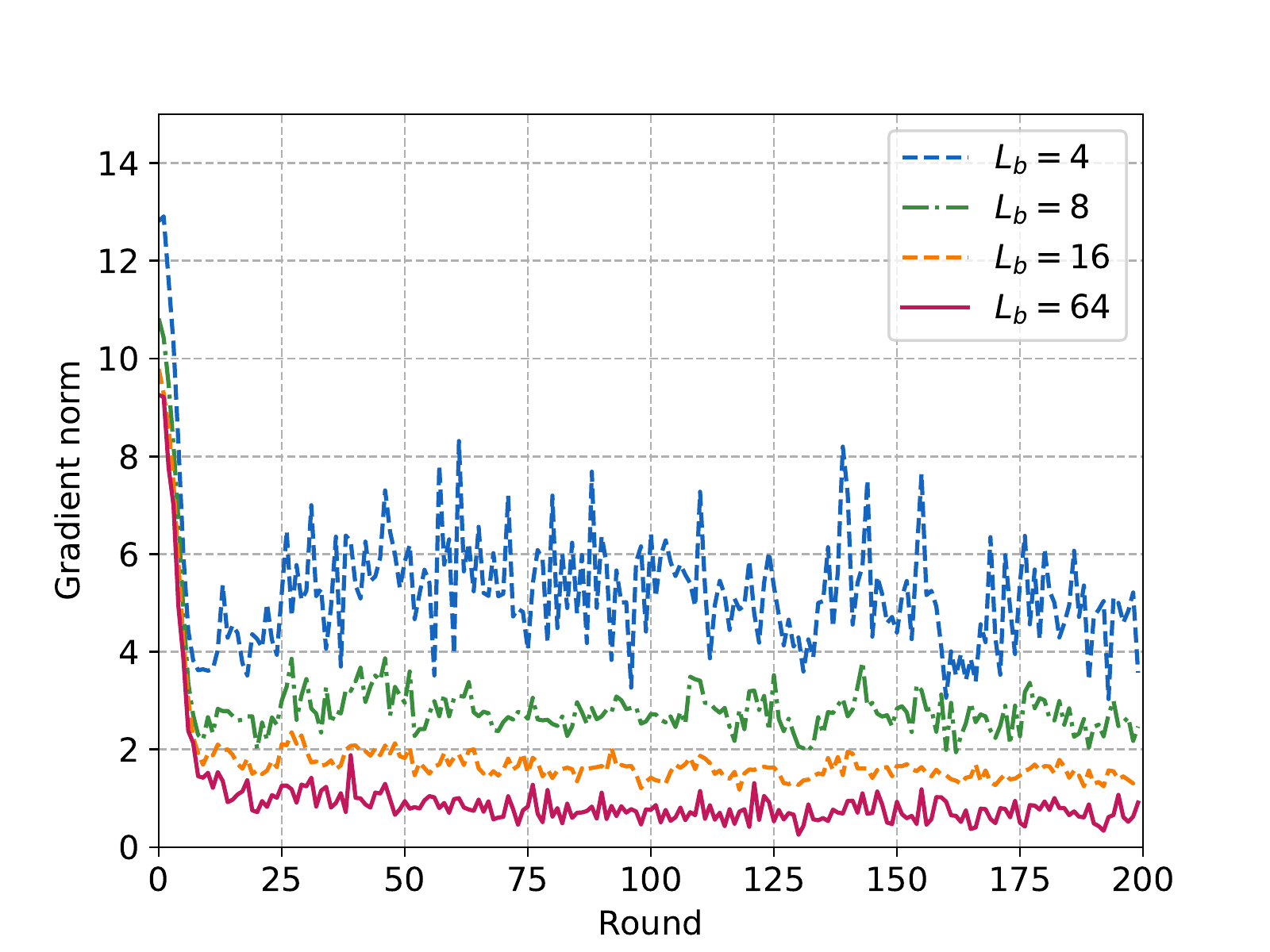} 	\vspace{-1mm}}  
	\vspace{-3mm}
	\caption{The $l_2$-norm of local gradients and their estimated values on the MNIST dataset.}
	\label{power_est_MNIST}
\end{figure*}

\begin{figure*}[!t]
	\centering	
	\vspace{-8mm}
	\subfigure[I.i.d. local data.]{\label{power_est_CIFAR_iid}			
		\includegraphics[width=0.4\textwidth]{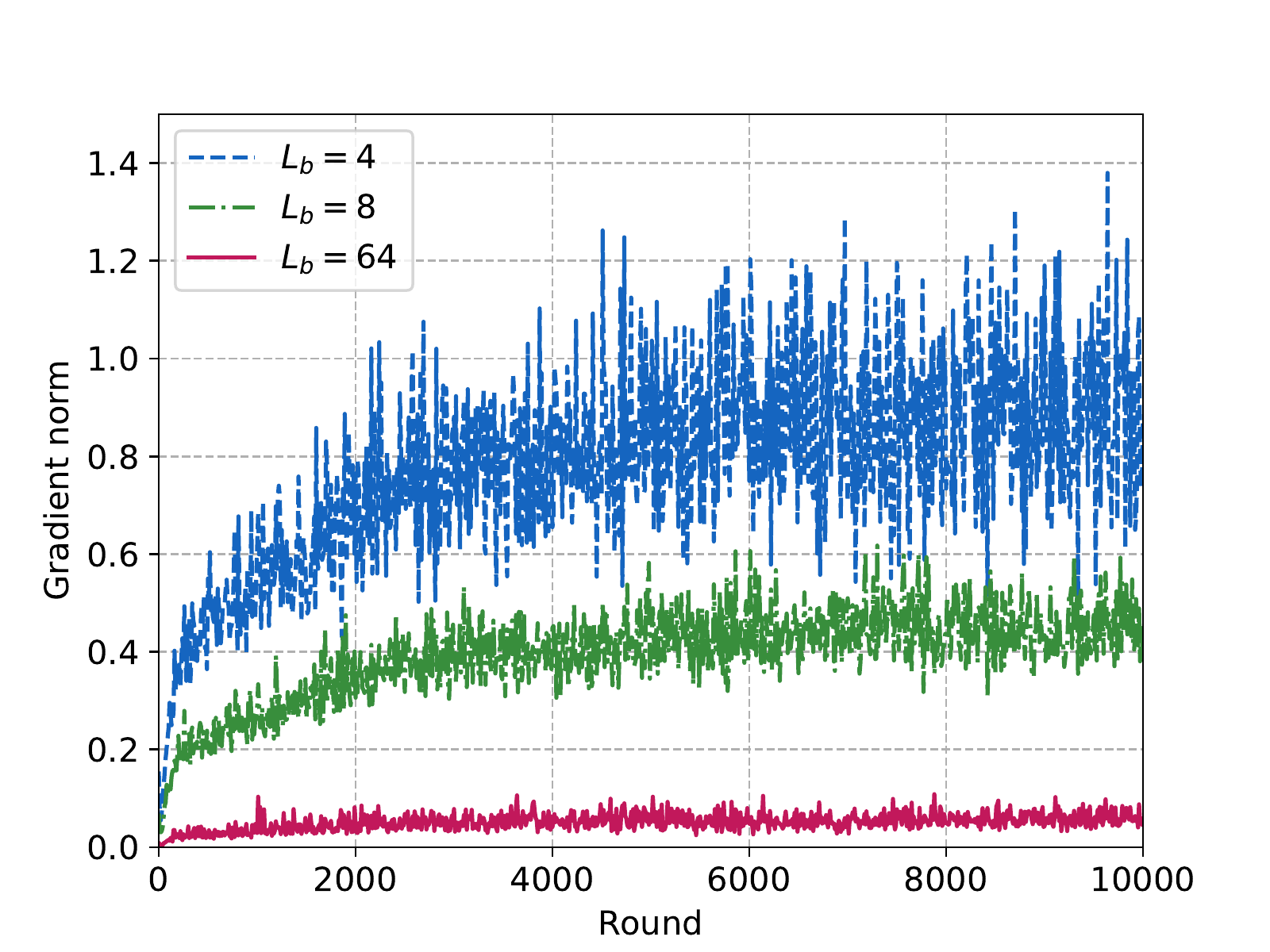}}		
	\hspace{8mm}
	\subfigure[Non-i.i.d. ($m=2$).]{\label{power_est_CIFAR_label2}	
		\includegraphics[width=0.4\textwidth]{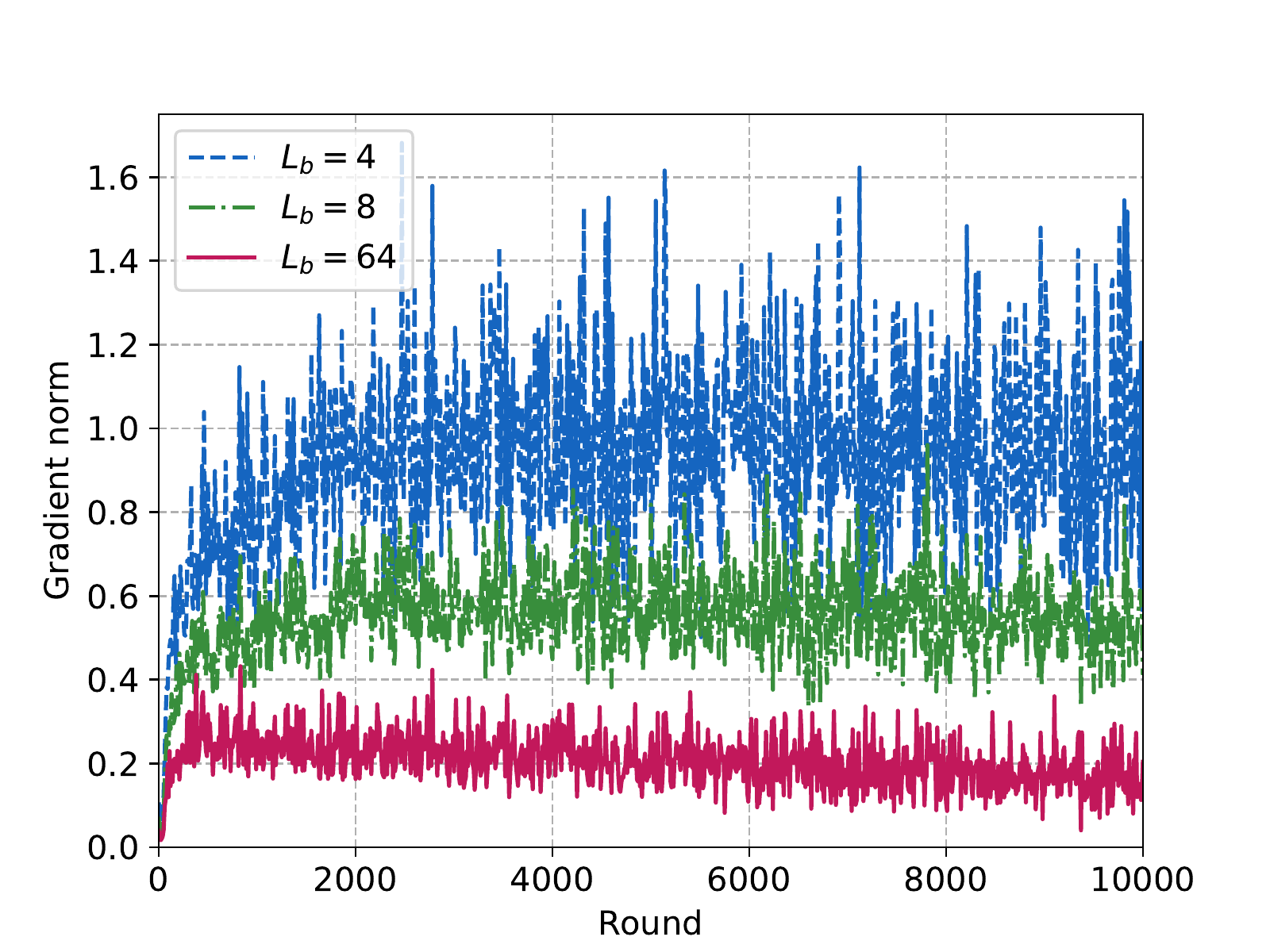}}  
	\vspace{-3mm}
	\caption{The $l_2$-norm of local gradients and their estimated values on the CIFAR-10 dataset.}
	\label{power_est_CIFAR}
\end{figure*}

In Fig. \ref{power_est_MNIST} and Fig. \ref{power_est_CIFAR}, we first evaluate the EST-C and EST-P methods proposed in Section \ref{subsec_esti} that estimate the $l_2$-norm of local gradient, by observing the temporal variations of the gradients. To eliminate the impact of device scheduling, we do not limit the energy consumption and all devices are scheduled. The batch size $L_b$ used for the model training is $64$. In each round, each device further computes its local gradient with smaller batch sizes $L_b=4$, $8$ and $16$ and records the corresponding estimated gradient norm, which is adopted by the EST-C method as $\left\lVert\tilde{\g}^\text{[est]}_{n,t}\right\rVert_2^2$. 
For the EST-P method, $\left\lVert\tilde{\g}^\text{[est]}_{n,t}\right\rVert_2^2$ will be given the value of the $l_2$-norm of gradients with $L_b=64$ at a certain round before $t$. Each curve is averaged over $50$ and $20$ runs for MNIST and CIFAR-10, respectively.

As shown in Fig. \ref{power_est_MNIST} and Fig. \ref{power_est_CIFAR}, the gradient norms achieved by different batch sizes are highly varying, and a smaller batch size yields a higher $l_2$-norm of gradient due to the non-negligible gradient variance, which is consistent with the analysis in \eqref{est-c-dev}. This result indicates that the EST-C method cannot provide an accurate estimation of gradient norm. 
Meanwhile, with a fixed batch size, such as $L_b=64$, the gradient norm has a strong temporal correlation. Therefore, the EST-P method can provide a much better estimate of the gradient norm, which is embedded in the proposed dynamic device scheduling algorithm.

\subsection{Performance of the Proposed Device Scheduling Algorithm}

\begin{figure*}[!t]
	\centering	
	\vspace{-3mm}
	\subfigure[Accuracy on test dataset.]{\label{algo_MNIST_acc}			
		\includegraphics[width=0.56\textwidth]{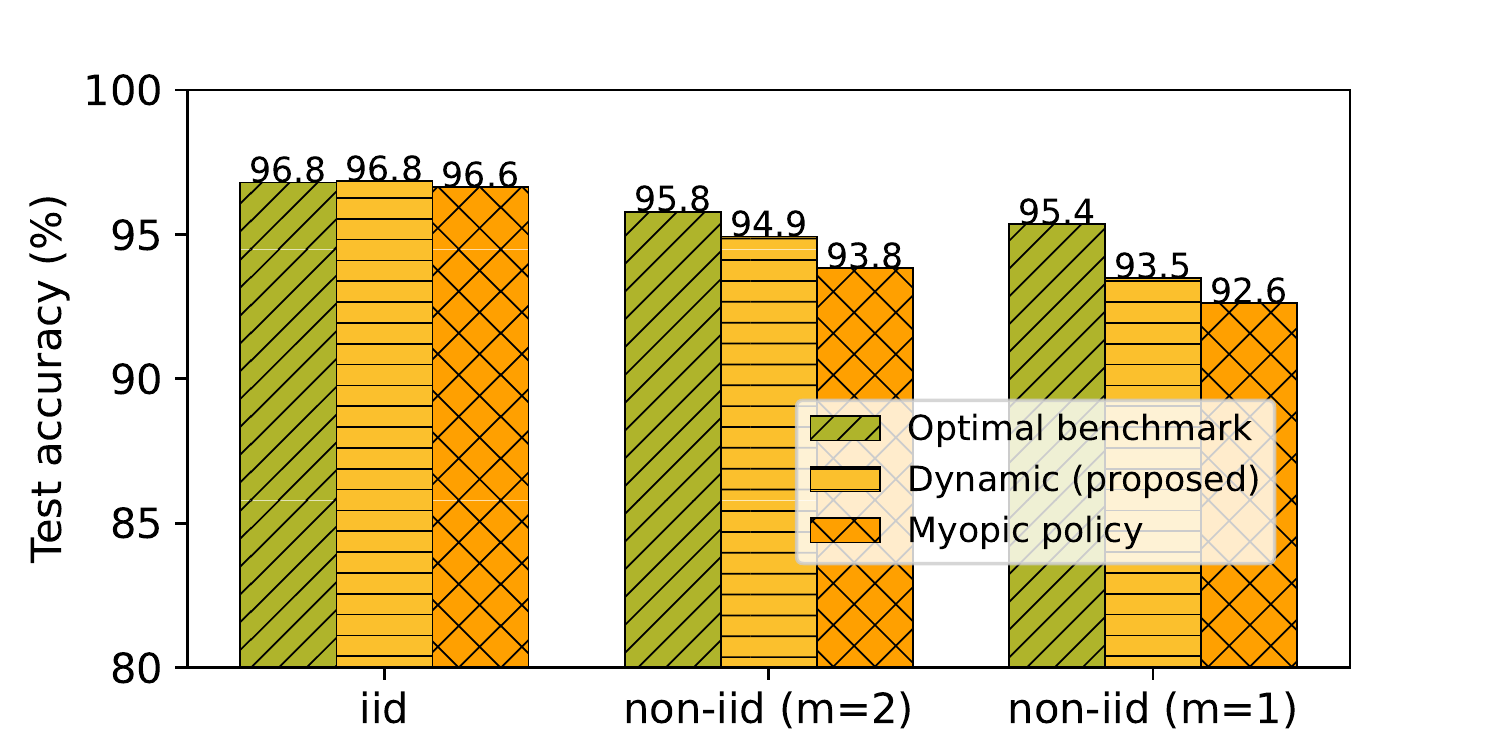}}		
	\subfigure[Unified cumulative energy usage.]{\label{algo_MNIST_energy}	
		\includegraphics[width=0.38\textwidth]{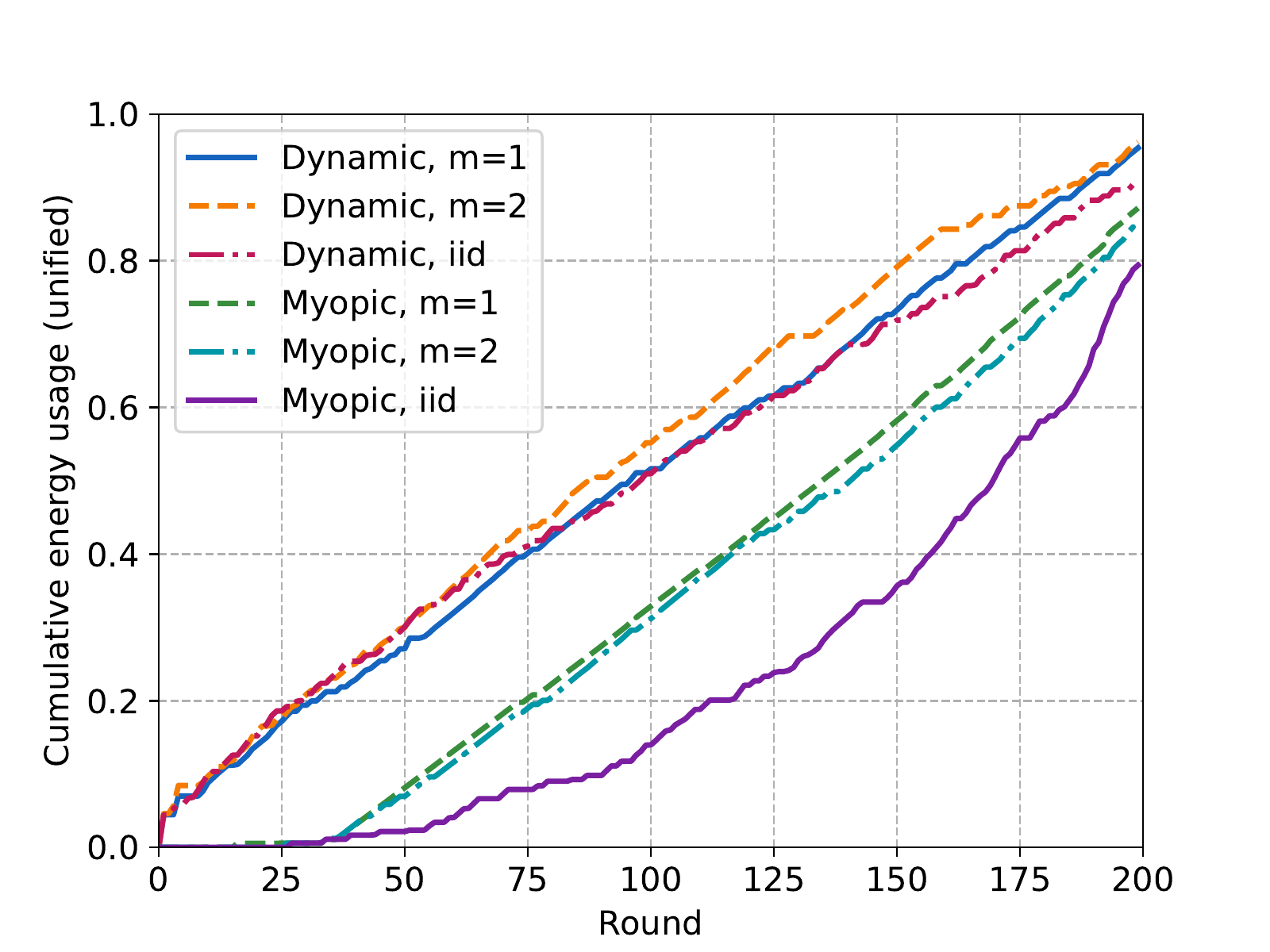}}  
	\vspace{-3mm}
	\caption{Performance of the proposed dynamic scheduling algorithm and benchmarks on MNIST. }
	\label{algo_MNIST}
\end{figure*}

\begin{figure*}[!t]
	\centering	
	\vspace{-8mm}
	\subfigure[Accuracy on test dataset.]{\label{algo_CIFAR_acc}			
		\includegraphics[width=0.56\textwidth]{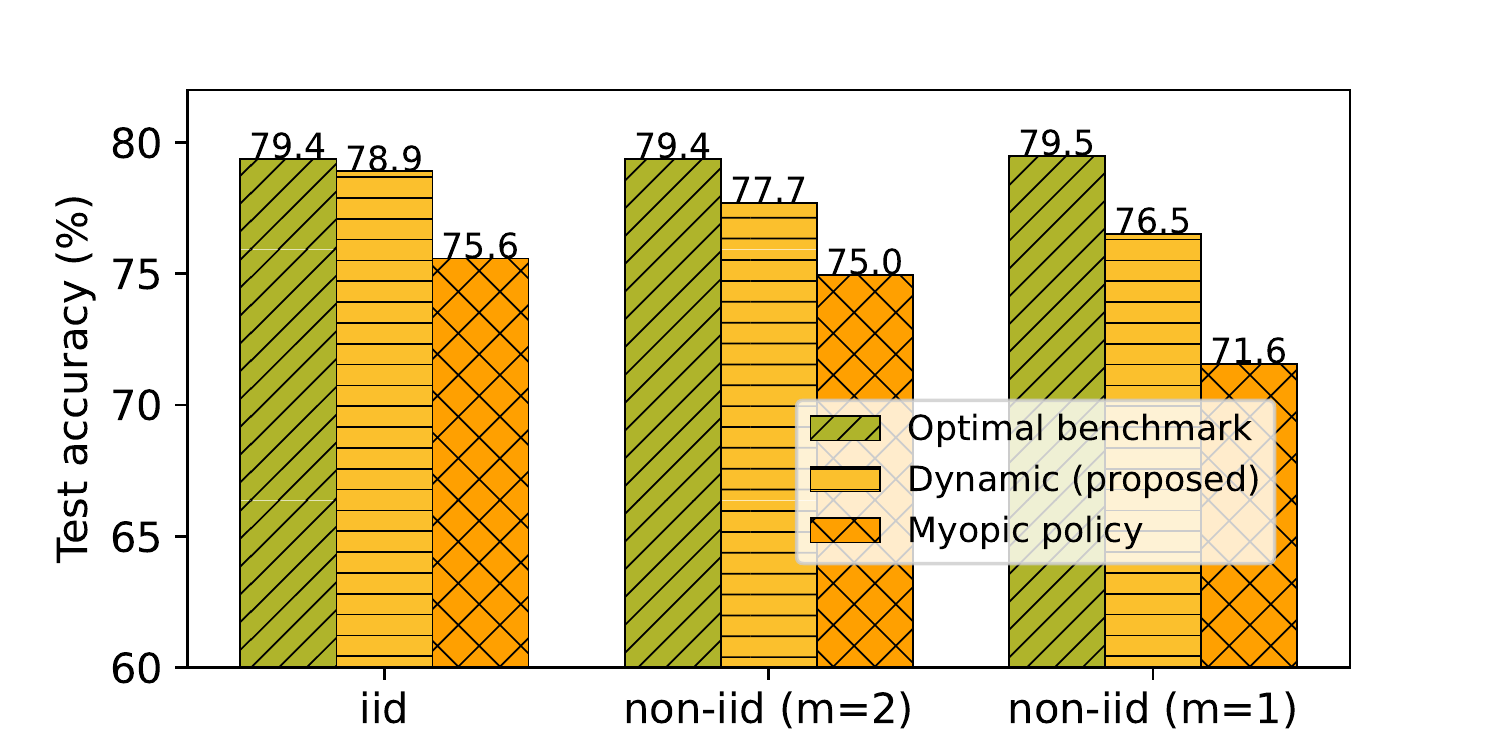}}		
	\subfigure[Unified cumulative energy usage.]{\label{algo_CIFAR_energy}	
		\includegraphics[width=0.38\textwidth]{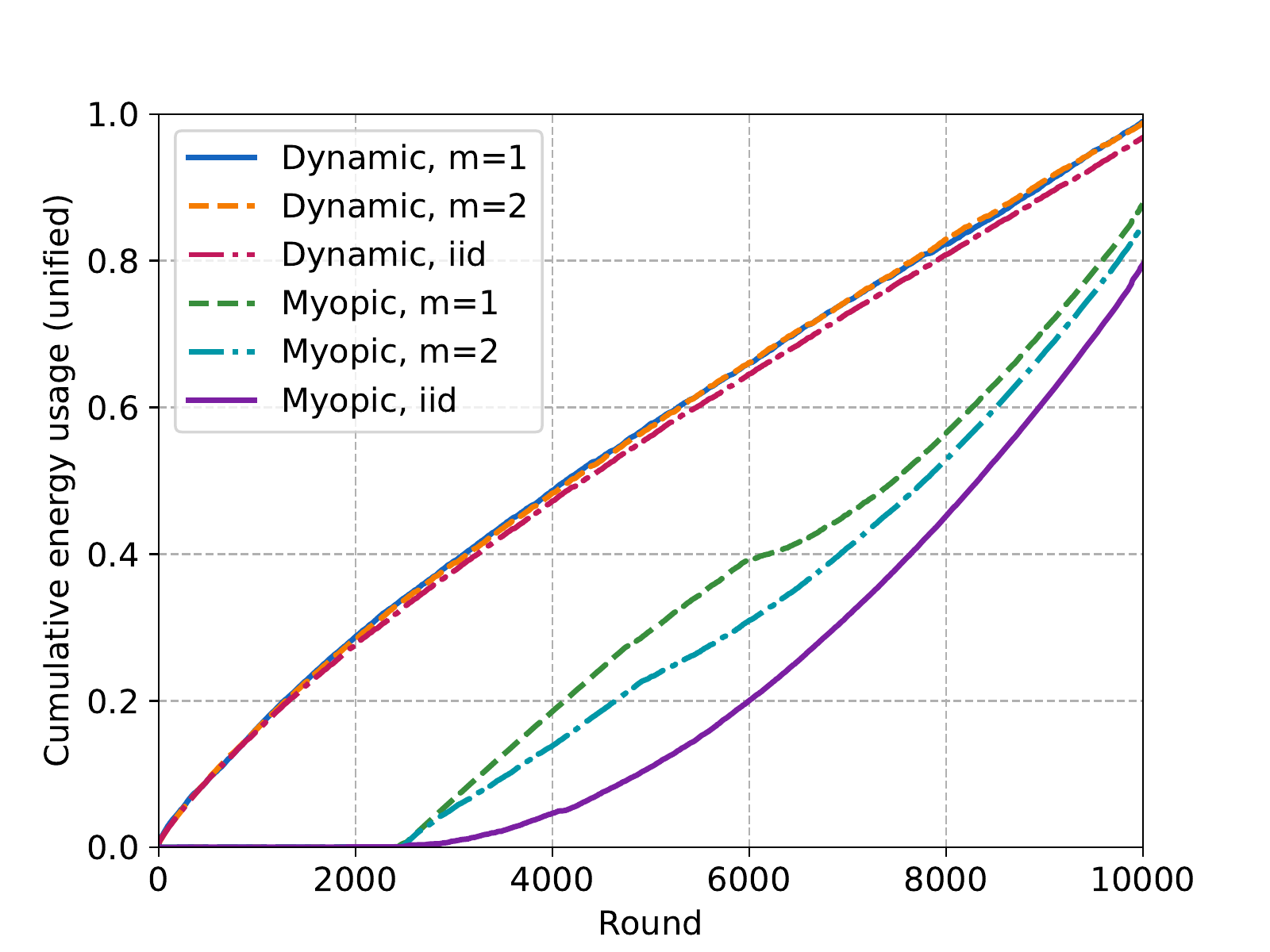}}  
	\vspace{-3mm}
	\caption{Performance of the proposed dynamic scheduling algorithm and benchmarks on CIFAR-10. }
	\label{algo_CIFAR}
\end{figure*}
We compare the performance of the proposed scheduling algorithm with two benchmarks:

\emph{1) Optimal benchmark:} Devices do not have energy limitations, so that all of them participate in each training round.

\emph{2) Myopic policy:} For each device $n$, the maximum energy that can be used in round $t$ is given by the remaining energy divided by the remaining number of rounds, i.e., $\frac{\bar{E}_n-\sum_{\tau=1}^{t-1}\beta_{n,t}E_{n,t}}{T-t+1}$.

In Fig. \ref{algo_MNIST}, we compare the training performance and energy consumption of the proposed dynamic scheduling algorithm with the optimal and myopic benchmarks on MNIST. Let $\bar{E}=1\mathrm{J}$ be the energy budget per round, and the total energy budget of each device is $\bar{E}_n=T\bar{E}, \forall n$. For non-i.i.d. data with 1 label per device, the weight parameter is $V=5\times10^7$, while for the other two cases, $V=10^8$.
The training performance is characterized by the accuracy of the MLP model on the test dataset, as shown in Fig. \ref{algo_MNIST_acc}. Results show that our proposed dynamic scheduling algorithm achieves the optimal accuracy under i.i.d. data, and always outperforms the myopic policy.
The maximum value of the unified cumulative energy usage across devices till the $t$-th round, given by $\max_{n\in\mathcal{N}} ~\frac{\sum_{\tau=1}^{t} \beta_{n,\tau}E_{n,\tau} }{t\bar{E}}$, is plotted in Fig. \ref{algo_MNIST_energy}. For the myopic policy, the energy  required for computation and communication exceeds the budget at the beginning of training, thus no devices can be scheduled. However, our proposed algorithm enables devices to use energy in a more flexible way, thus improving the training performance. 

Similar comparison is made on CIFAR-10 dataset in Fig. \ref{algo_CIFAR}, where $\bar{E}=8\mathrm{J}$ and $V=5\times10^{11}$.
Note that compared to the local computation energy required per round ($10\mathrm{J}$), the energy budget is relatively limited, and the advantage of the proposed dynamic scheduling over the myopic policy is more prominent in such a scenario. In particular,  under the highly non-i.i.d. case with $m=1$, dynamic scheduling improves the accuracy by $4.9\%$ compared to the myopic policy, by utilizing $10\%$ more energy in a more balanced manner.
We can also see that our proposed algorithm can satisfy the energy constraints of devices under both datasets (at the end of training, the unified energy usage is smaller than 1).

In the following, we further explore the impact of key parameters on the training performance and energy consumption with CIFAR-10, as it is more challenging than MNIST. We focus on the non-i.i.d. case, where each device has a local subset with $m=2$ labels.

Fig. \ref{V_CIFAR} validates that the weight parameter $V$ can balance the trade-off between the training performance and energy consumption, where $\bar{E}=8\mathrm{J}$. As $V$ increases, devices use energy in a more aggressive manner, leading to a higher energy usage and more scheduled devices, so as to accelerate the convergence. However, if $V$ is too large, such as $V=10^{12}$, energy is not given enough attention and finally the limit is violated. In practical systems, $V$ should be judiciously selected to optimize the training performance while satisfying the energy constraints.

\begin{figure*}[!t]
	\centering	
	\vspace{-3mm}
	\hspace{-8mm}
	\subfigure[Accuracy on test dataset.]{\label{V_CIFAR_acc}			
		\includegraphics[width=0.36\textwidth]{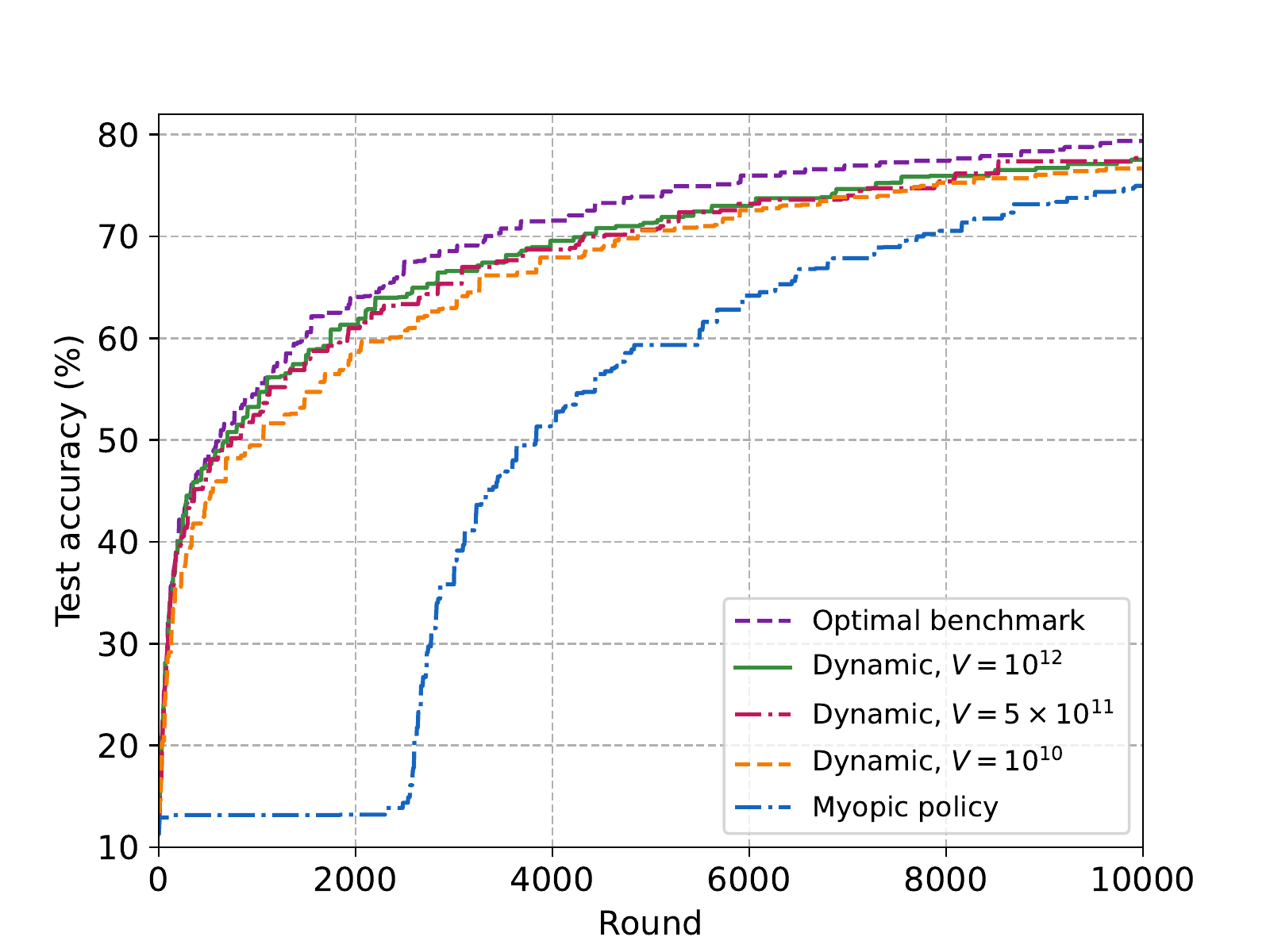}}			
	\hspace{-8mm}
	\subfigure[Unified cumulative energy usage.]{\label{V_CIFAR_energy}	
		\includegraphics[width=0.36\textwidth]{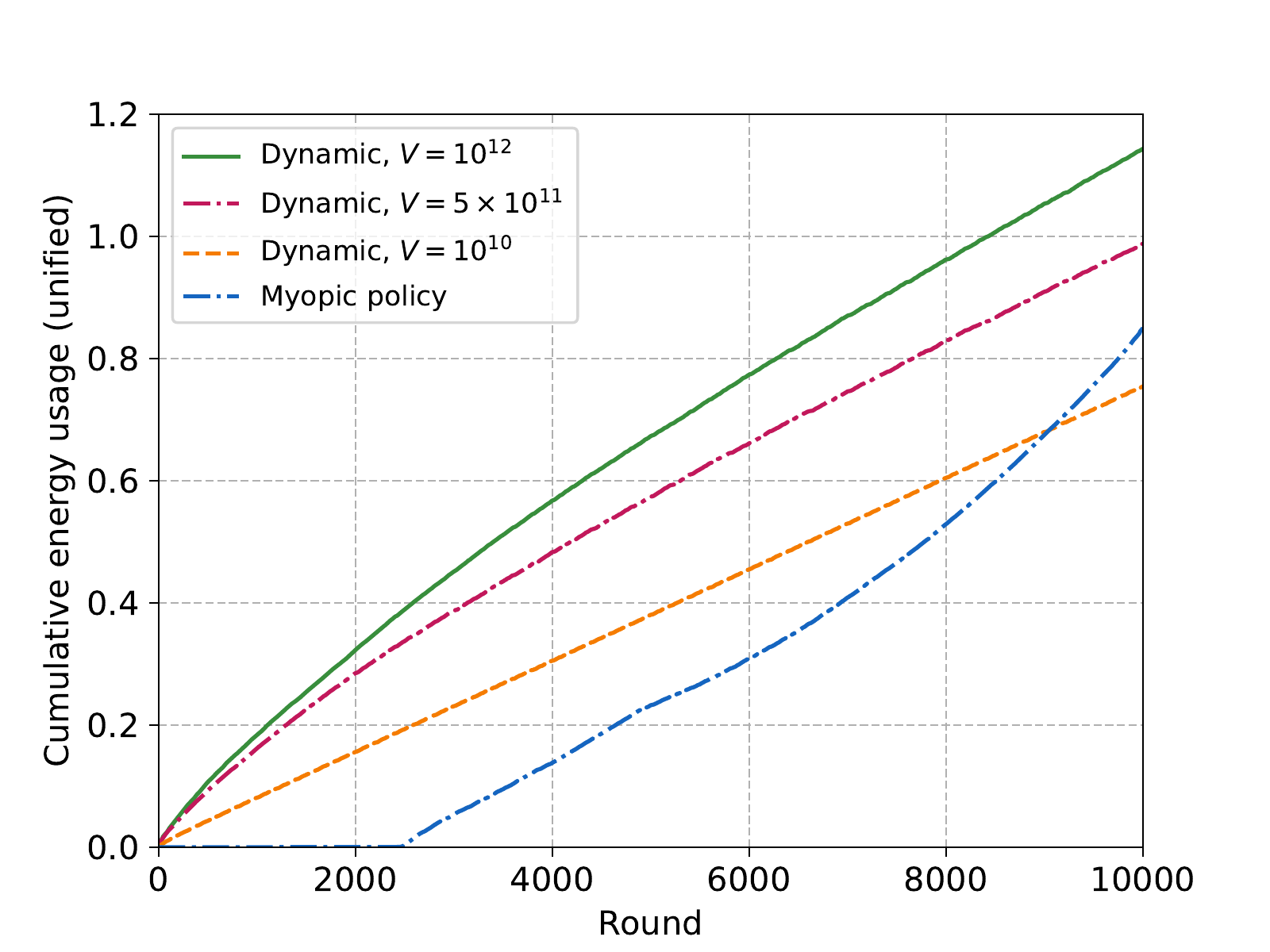}}  
	\hspace{-8mm}
	\subfigure[Scheduled devices.]{\label{V_CIFAR_frac}	
		\includegraphics[width=0.36\textwidth]{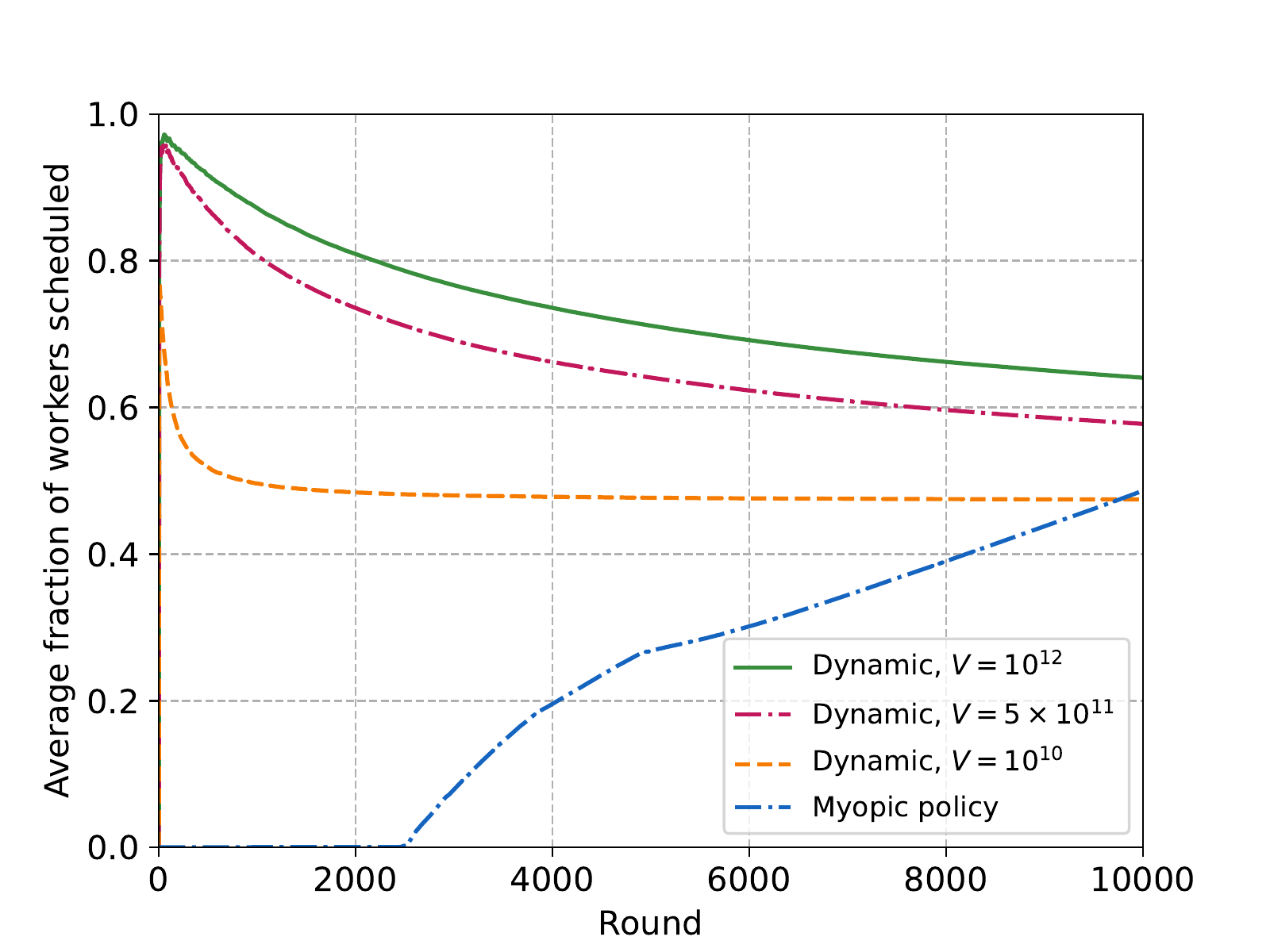}} 
	\hspace{-8mm} 
	\vspace{-3mm}
	\caption{Performance of the proposed algorithm under different weight parameter $V$ on CIFAR-10.}
	\label{V_CIFAR}
\end{figure*}

\begin{figure*}[!t]
	\centering	
	\vspace{-5mm}
	\hspace{-8mm}	
	\subfigure[Accuracy on test dataset.]{\label{SNR_CIFAR_acc}			
		\includegraphics[width=0.36\textwidth]{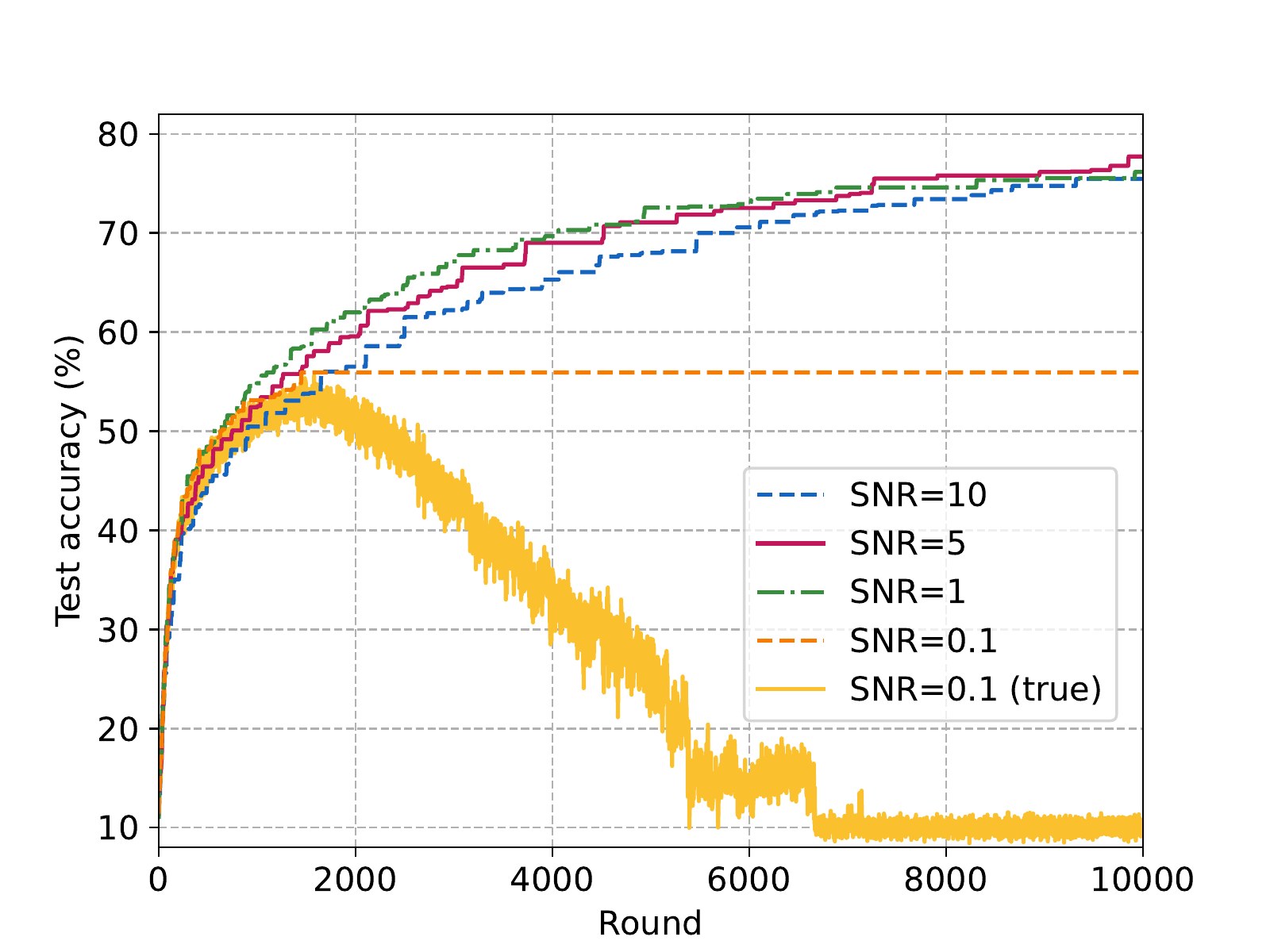}}	
	\hspace{-8mm}	
	\subfigure[Unified cumulative energy usage.]{\label{SNR_CIFAR_energy}	
		\includegraphics[width=0.36\textwidth]{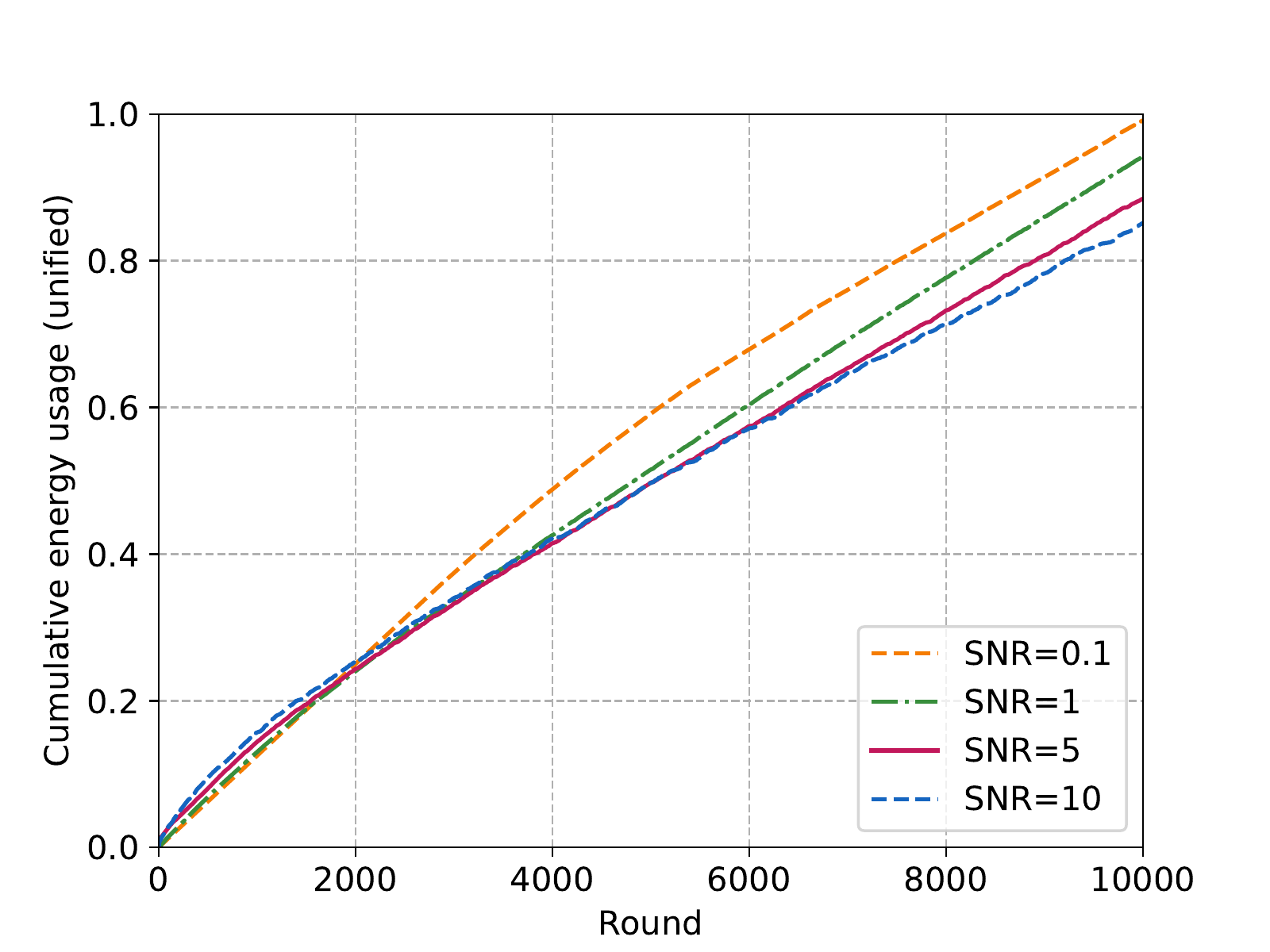}}  
	\hspace{-8mm}
	\subfigure[Scheduled devices.]{\label{SNR_CIFAR_frac}	
		\includegraphics[width=0.36\textwidth]{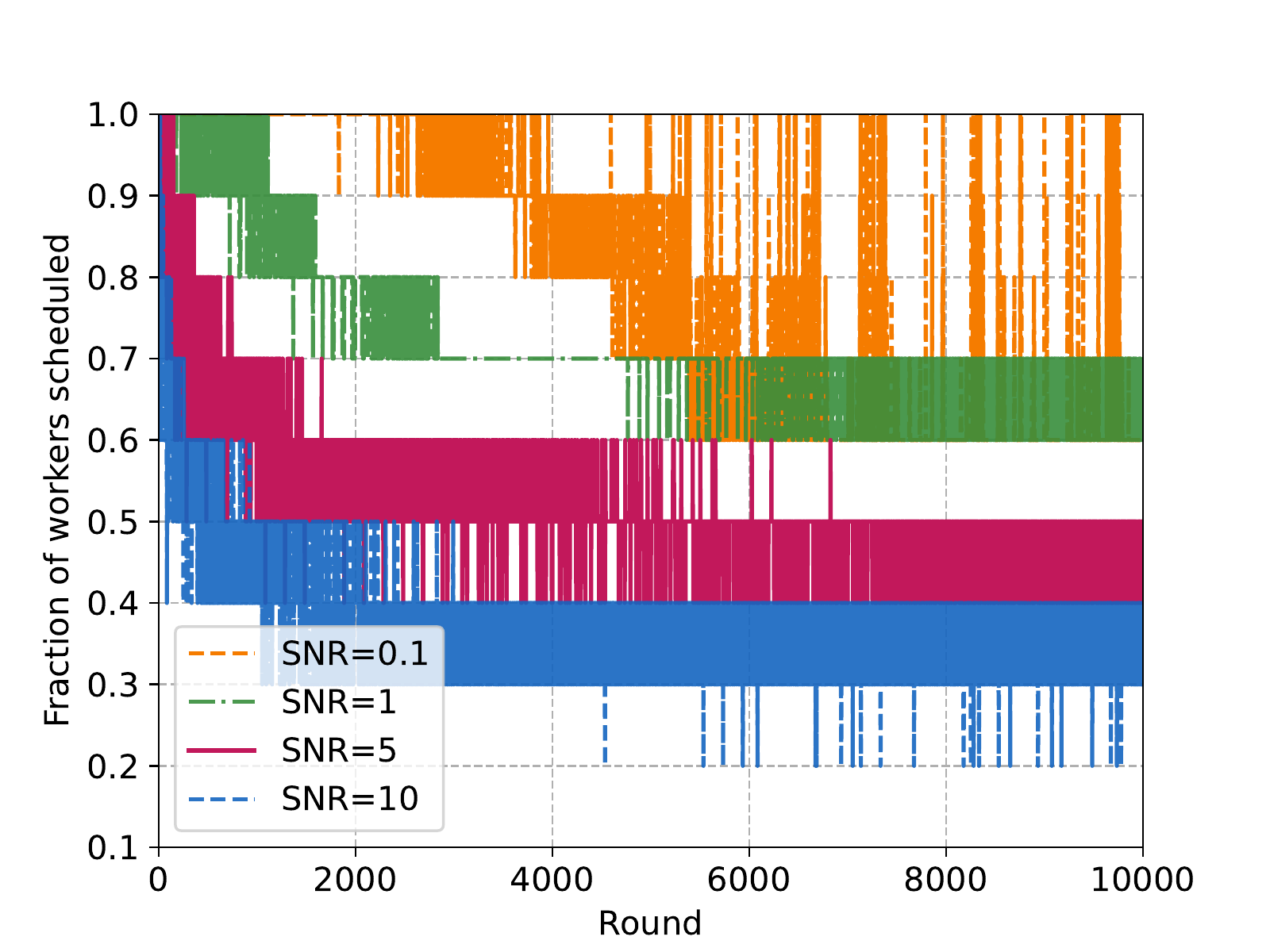}}  
	\hspace{-8mm}
	\vspace{-3mm}
	\caption{Performance of the proposed algorithm under different received SNR thresholds on CIFAR-10. }
	\label{SNR_CIFAR}
\end{figure*}

The impact of the received SNR threshold $\gamma_0$ on the training performance and energy consumption with the proposed dynamic scheduling algorithm is shown in Fig. \ref{SNR_CIFAR}, where $\bar{E}=8\mathrm{J}$, and $V=2.5\times10^{11}$. In Fig. \ref{SNR_CIFAR_acc}, the curve marked with `true' plots the model accuracy on the test dataset in the current round, with SNR threshold $0.1$, while the other curves present the best test accuracy up-to-date. The maximum cumulative energy usage and instantaneous fraction of devices that are scheduled in each round is shown in Fig. \ref{SNR_CIFAR_energy} and Fig. \ref{SNR_CIFAR_frac}, respectively.
Clearly, a smaller SNR threshold helps to save communication energy, and thus more devices can be scheduled in each round. However, the cumulative noise might degrade the accuracy or even diverge the training if the SNR is too low, for instance when $\text{SNR}=0.1$. On the other hand, a larger SNR, such as $\text{SNR}=10$, makes communication more energy-consuming, which also degrades the training performance due to fewer participants. A proper value of the received SNR threshold should be given to balance the negative impact of noise and the energy consumption. As shown in Fig. \ref{SNR_CIFAR}, $\text{SNR}=5$ is the best choice under our simulation setting.

\begin{figure*}[!t]
	\centering	
	\subfigure[Accuracy on test dataset.]{\label{Ebar_CIFAR_acc}			
		\includegraphics[width=0.56\textwidth]{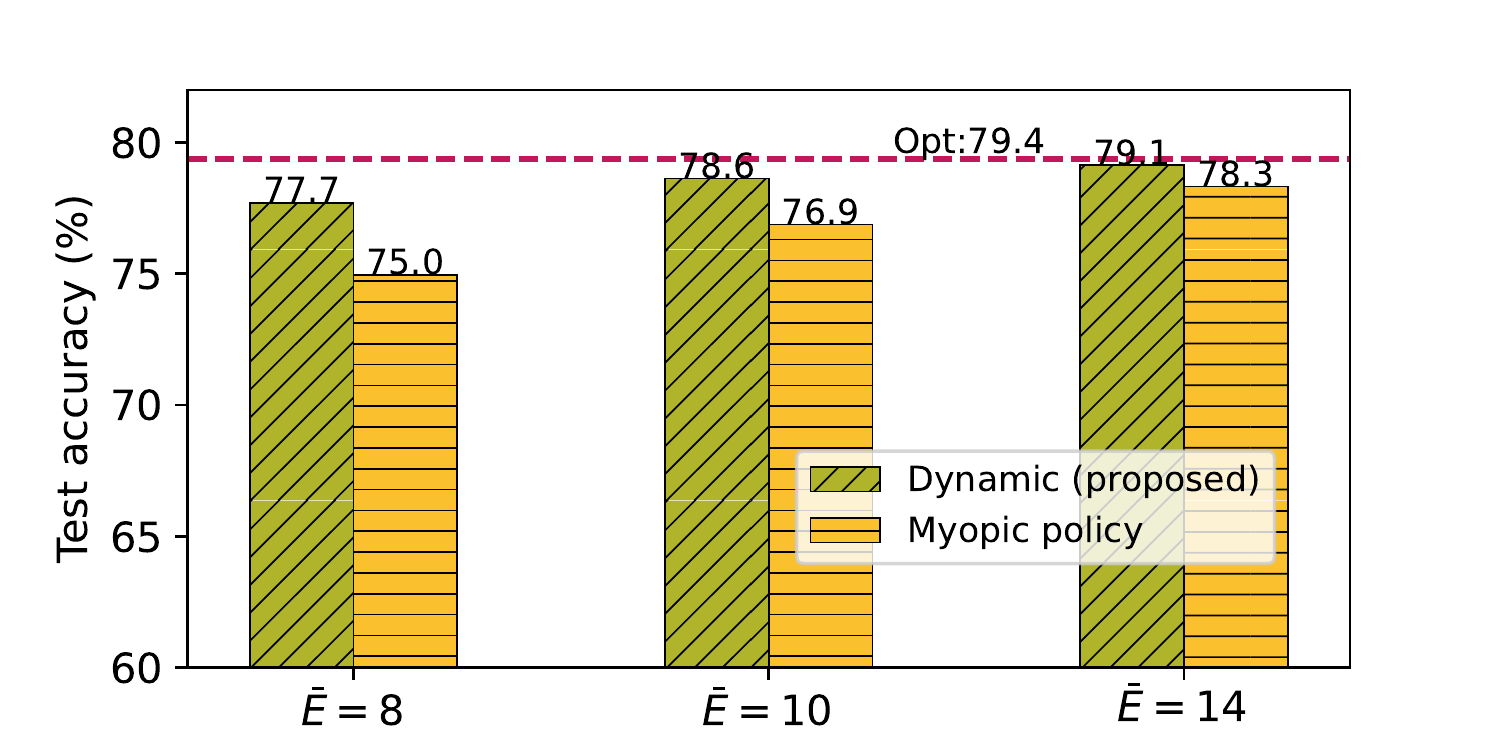}}		
	\subfigure[Unified cumulative energy usage.]{\label{Ebar_CIFAR_energy}	
		\includegraphics[width=0.38\textwidth]{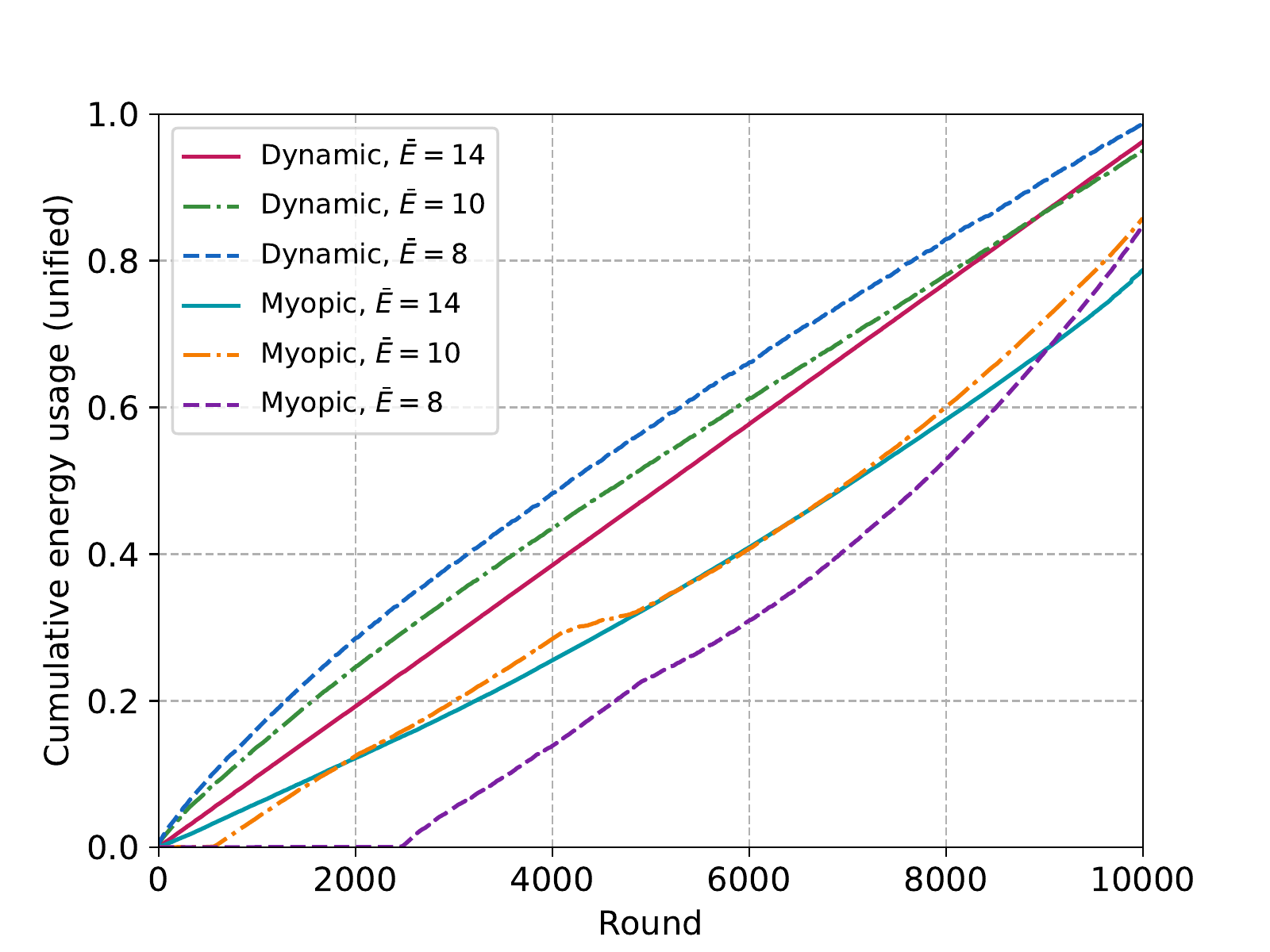}}  
	\vspace{-3mm}
	\caption{Performance of the proposed algorithm and benchmarks under different energy budgets on CIFAR-10. }
	\label{Ebar_CIFAR}
\end{figure*}

\begin{figure*}[!t]
	\centering	
	\vspace{-5mm}			
	\includegraphics[width=0.5\textwidth]{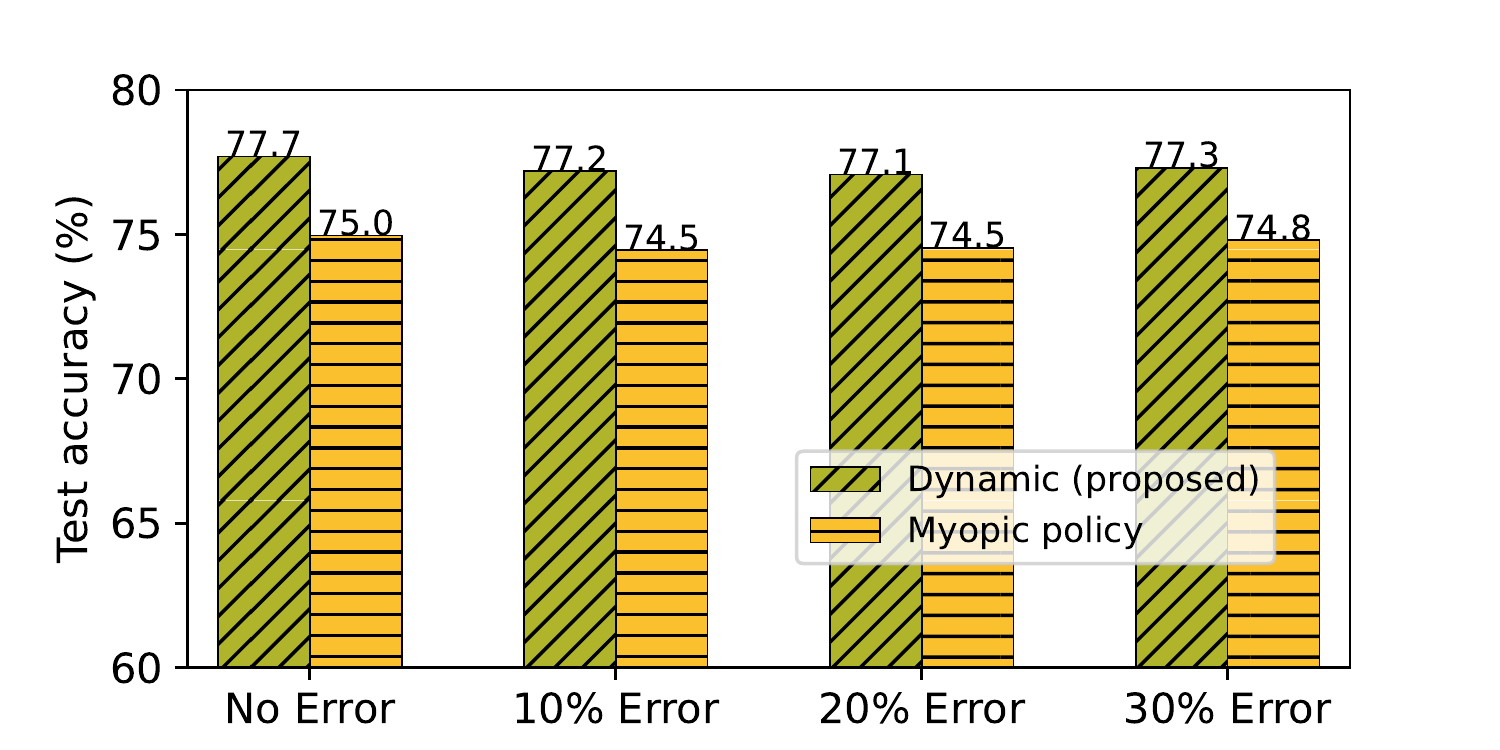}
	\vspace{-5mm}	
	\caption{Performance of the proposed algorithm and myopic benchmark under different channel estimation errors on CIFAR-10. }
	\label{Herror_CIFAR}
\end{figure*}

We compare our proposed algorithm with optimal and myopic benchmarks under different energy budgets in Fig. \ref{Ebar_CIFAR}. For $\bar{E}=14\mathrm{J}$, we set $V=10^{10}$, and for $\bar{E}=8\mathrm{J}$ or $10\mathrm{J}$, we let $V=5\times10^{11}$.
Our proposed dynamic scheduling algorithm always outperforms the myopic benchmark by achieving higher accuracy and utilizing energy more efficiently, and approaches the optimal accuracy as $\bar{E}$ increases. Moreover, the accuracy gap between the proposed algorithm and myopic policy is $2.7\%$, $1.7\%$ and $0.8\%$ for $\bar{E}=8$, $10$ and $14$, respectively, indicating that the dynamic scheduling algorithm is particularly promising under the energy-limited regime.

Finally, we evaluate the robustness of the proposed dynamic scheduling algorithm by introducing channel observation errors, where $\bar{E}=8\mathrm{J}$ and $V=5\times10^{11}$. In Fig. \ref{Herror_CIFAR}, the first group of bars are obtained without channel observation error, i.e., $\tilde{h}_{n,t}=h_{n,t}$. The second to the forth group of bars suffer inaccurate channel observations. For example, if the error is $20\%$, then $\tilde{h}_{n,t}$ is uniformly distributed within $[0.8h_{n,t}, 1.2h_{n,t}]$. A larger observation error further leads to less accurate energy estimations $\tilde{E}_{n,t}$. However, the training performance of the proposed dynamic scheduling algorithm only suffers a tiny degradation, which validates its robustness in practical scenarios. We also mention that the myopic policy also performs well under different observation errors compared to the error-free case. Nevertheless, the proposed algorithm still beats the myopic policy by a significant margin in all the scenarios.

\section{Conclusions} \label{con}

We have investigated the device scheduling problem for FEEL with over-the-air gradient aggregation, aiming to optimize the training performance under joint communication and computation energy limits of devices. Convergence analysis has been carried out showing the importance of device participation to the training performance, and an energy-aware dynamic device scheduling algorithm has been developed. 
In particular, we have noticed the existence of unobservable states, mainly the $l_2$-norm of local gradients, for online decision making in over-the-air FEEL, and proposed an estimated-drift-plus-penalty solution based on the Lyapunov optimization framework accordingly.
We have characterized a theoretical guarantee for the proposed dynamic scheduling algorithm by taking the deviation of estimated states into consideration.
Experiments on MNIST and CIFAR-10 datasets have been carried out to validate the theoretical findings. Compared to the myopic benchmark, we have shown a significant $4.9\%$ accuracy improvement on CIFAR-10 for a highly non-i.i.d. data distribution and stringent energy constraints.

As future directions, heterogeneous data distributions across devices can be considered, where local datasets represent different number of classes. We would like to observe if data diversity of a device should be taken into account for the scheduling decision.
The trade-off between training delay and energy consumption in over-the-air FEEL is also worth further investigation.

\appendices{}

\section{Proof of Lemma 1} \label{a1}
For the simplicity of notation, let $\tilde{\g}_t\triangleq\frac{\sum_{n\in\mathcal{B}_t}\tilde{\g}_{n,t}}{|\mathcal{B}_t|}$, $\tilde{\z}_t\triangleq\frac{\z_t}{\sigma_t |\mathcal{B}_t|}$. 
%
Thus the global model is updated according to
\begin{align}
\w_{t}=\w_{t-1}-\eta_t(\tilde{\g}_t+\tilde{\z}_t).
\end{align}

According to Assumption \ref{ass_smooth}, the gap of loss between two adjacent rounds can be bounded by
\begin{align}
F(\w_{t})-F(\w_{t-1})&\leq \nabla F(\w_{t-1})^{\text{T}}(\w_{t}-\w_{t-1})+\frac{l}{2}\lVert \w_{t}-\w_{t-1}\rVert_2^2 \\
&=  -\eta_t \g_t^{\text{T}}  (\tilde{\g}_t+\tilde{\z}_t) + \frac{l\eta_t^2}{2} \lVert \tilde{\g}_t+\tilde{\z}_t \rVert_2^2.
\end{align}

Recall that each entry in the noise vector $\z_t$ follows Gaussian distribution with zero mean and variance $\sigma_0^2$. Taking the expectation over noise, and considering the fact that the channel noise and local gradient are independent, we can obtain
\begin{align}
\mathbb{E}_{\z_t}\left[\lVert \tilde{\g}_t+\tilde{\z}_t \rVert_2^2\right]=\lVert \tilde{\g}_t\rVert_2^2+ \mathbb{E}_{\z_t}\left[\lVert \tilde{\z}_t\rVert_2^2\right]
=\lVert \tilde{\g}_t\rVert_2^2+\frac{\sigma_0^2s}{\sigma_t^2 |\mathcal{B}_t|^2},
\end{align}
\begin{align}  \label{a1_3}
\mathbb{E}_{\z_t}\left[F(\w_{t})-F(\w_{t-1})\right] \leq -\eta_t \g_t^{\text{T}} \tilde{\g}_t+ \frac{l\eta_t^2}{2} \lVert \tilde{\g}_t \rVert_2^2 +  \frac{l\eta_t^2}{2}\frac{\sigma_0^2s}{\sigma_t^2 |\mathcal{B}_t|^2}.
\end{align}

Taking the expectation over stochastic data sampling, and based on Assumption \ref{ass_unbias}, we get
\begin{align}  \label{a1_4}
\mathbb{E}_{\mathcal{L}_{n,t}} \left[\tilde{\g}_t\right]= 
\mathbb{E}_{\mathcal{L}_{n,t}} \left[\frac{\sum_{n\in\mathcal{B}_t}\tilde{\g}_{n,t}}{|\mathcal{B}_t|}\right]=\g_t,
\end{align}
\begin{align} \label{a1_2}
\mathbb{E}_{\mathcal{L}_{n,t}}\left[\lVert \tilde{\g}_t \rVert_2^2\right]=\mathbb{E}\left[\left\lVert \frac{\sum_{n\in\mathcal{B}_t}\sum_{\x \in\mathcal{L}_{n,t}} \nabla f\left(\w_{t-1},\x\right)}{L_b |\mathcal{B}_t|}\right\rVert_2^2\right]\leq \lVert \g_t\rVert_2^2+\frac{G^2}{L_b |\mathcal{B}_t|}.
\end{align}

Finally, taking the expectation over noise and SGD on the left hand side of \eqref{a1_3}, and substituting its right hand side with \eqref{a1_4} and \eqref{a1_2}, we have 
\begin{align}
\mathbb{E}[F(\w_{t})-F(\w_{t-1})]&\leq -\eta_t \lVert \g_t \rVert_2^2 + \frac{l\eta_t^2}{2} \left( \lVert \g_t \rVert_2^2 +\frac{G^2}{L_b |\mathcal{B}_t|}\right) +  \frac{l\eta_t^2}{2}\frac{\sigma_0^2s}{\sigma_t^2 |\mathcal{B}_t|^2} \nonumber\\
&= -\eta_t \left(1-\frac{l\eta_t}{2}\right) \lVert \g_t\rVert_2^2 +\frac{l\eta_t^2}{2}\left(\frac{G^2}{L_b |\mathcal{B}_t|}+\frac{\sigma_0^2 s}{\sigma_t^2|\mathcal{B}_t|^2}\right).
\end{align}

Since $\mathbb{E}[F(\w_{t})-F(\w_{t-1})]=\mathbb{E}[F(\w_{t})]-\mathbb{E}[F(\w_{t-1})]$, Lemma 1 is proved.


\section{Proof of Theorem 1} \label{a2}
By the $\mu$-strong convexity of the loss functions (Assumption \ref{ass_convex}), the Polyak-Lojasiewicz inequality holds
\begin{align} \label{g_t_bound}
\lVert \g_t\rVert_2^2 \geq 2\mu (F(\w_{t-1})-F^*).
\end{align}

Substituting \eqref{g_t_bound} into Lemma \ref{single_round_fix}, and assuming that $\eta_t\leq \frac{1}{l}$ (thus $1-\frac{l\eta_t}{2}\geq\frac{1}{2}$), we can obtain
\begin{align} 
\mathbb{E}[F(\w_{t})-F(\w_{t-1})]
&\leq -\eta_t \left(1-\frac{l\eta_t}{2}\right) \lVert \g_t\rVert_2^2 +
\frac{l\eta_t^2}{2}\left(\frac{G^2}{L_b |\mathcal{B}_t|}+\frac{\sigma_0^2 s}{\sigma_t^2|\mathcal{B}_t|^2}\right) \nonumber\\
& \leq  -\eta_t \mu (\mathbb{E}[F(\w_{t-1})]-F^*) +\frac{\eta_t}{2}\left(\frac{G^2}{L_b |\mathcal{B}_t|}+\frac{\sigma_0^2 s}{\sigma_t^2|\mathcal{B}_t|^2}\right). \label{a2_1}
\end{align} 

Let $A_t\triangleq\frac{\eta_t}{2}\left(\frac{G^2}{L_b |\mathcal{B}_t|}+\frac{\sigma_0^2 s}{\sigma_t^2|\mathcal{B}_t|^2}\right)$, and thus \eqref{a2_1} can be re-written as
\begin{align}
\mathbb{E}[F(\w_{t})]-F^*
& \leq  (1-\mu \eta_t ) (\mathbb{E}[F(\w_{t-1})]-F^*) +A_t .
\end{align}

With recursion, we can prove Theorem 1:
\begin{align}
\mathbb{E}[F(\w_{t})]-F^*
& \leq  (1-\mu \eta_t ) (\mathbb{E}[F(\w_{t-1})]-F^*) +A_t \nonumber \\
& \leq  (1-\mu \eta_t ) (1-\mu \eta_{t-1} )(\mathbb{E}[F(\w_{t-2})]-F^*) +(1-\mu \eta_t ) A_{t-1} +A_t \nonumber \\
& \leq \cdots \leq (\mathbb{E}[F(\w_{0})]-F^*)\prod_{i=1}^{t} (1-\mu \eta_i )  +\sum_{i=1}^{t-1} A_i  \prod_{j=i+1}^{t}(1-\mu \eta_i ) +A_t.
\end{align}

\section{Proof of Theorem 2} \label{a3}
Let $y_{n,t}\triangleq \beta_{n,t}E_{n,t}-\frac{\bar{E}_n}{T}$, and $\tilde{y}_{n,t}\triangleq\beta_{n,t}\tilde{E}_{n,t}-\frac{\bar{E}_n}{T}$. Define the error of estimated energy consumption at device $n$ in the $t$-th round as $\delta_{n,t}\triangleq\beta_{n,t}\tilde{E}_{n,t}-\beta_{n,t}E_{n,t}=\tilde{y}_{n,t}-y_{n,t}$, with maximum absolute value $\delta_0\triangleq\max_{\{n,t\}}\left\{\left|\tilde{E}_{n,t}-E_{n,t}\right|\right\}$.
According to the evolution of the virtual queue, which is defined in \eqref{queue}, it is easy to prove that $q_{n,t+1}^2 \leq \left(q_{n,t}+y_{n,t}\right)^2$ and $y_{n,t}\leq q_{n,t+1}-q_{n,t}$.

Define the Lyapunov function as $L(t)\triangleq\sum_{n=1}^{N}\frac{1}{2} q_{n,t}^2$, and the Lyapunov drift of a single round as $\Delta_1(t)\triangleq L(t+1)-L(t)$, which is given by
\begin{align} \label{drift}
\Delta_1(t)&= L(t+1)-L(t)=\sum_{n=1}^{N}\left(\frac{1}{2}q_{n,t+1}^2- \frac{1}{2}q_{n,t}^2\right) \nonumber\\
&\leq \sum_{n=1}^{N}\left(\frac{1}{2}y_{n,t}^2 + q_{n,t}y_{n,t}\right) \leq  \theta_0+  \sum_{n=1}^{N}q_{n,t}y_{n,t},
\end{align}
where $\theta_0\triangleq\sum_{n=1}^{N}\frac{1}{2}\theta_n^2$ and $\theta_n\triangleq\max_t \left\{|y_{n,t}|\right\}$. By adding $VU_t$ on both sides of \eqref{drift}, an upper bound on the single-round drift-plus-penalty function is given by
\begin{align} 
\Delta_1(t)+VU_t&\leq  \theta_0+  \sum_{n=1}^{N}q_{n,t}y_{n,t}+VU_t \nonumber\\
&=\theta_0+  \sum_{n=1}^{N}q_{n,t}\left(\beta_{n,t}E_{n,t}-\frac{\bar{E}_n}{T}\right)+VU_t \label{drift_plus_penalty_ori} \\
&=\theta_0+ \sum_{n=1}^{N} q_{n,t}\left(\tilde{y}_{n,t}-\delta_{n,t}\right)+VU_t \nonumber\\
&=\theta_0+ \sum_{n=1}^{N} q_{n,t}\left(\beta_{n,t}\tilde{E}_{n,t}-\delta_{n,t}-\frac{\bar{E}_n}{T}\right)+VU_t. \label{drift_plus_penalty1}
\end{align}

The classical drift-plus-penalty algorithm of Lyapunov optimization aims to minimize the upper bound of $\Delta_1(t)+VU_t$, as shown in \eqref{drift_plus_penalty_ori}. Since we do not have the exact value of $E_{n,t}$, we instead minimize the estimated-drift-plus-penalty, as shown in \eqref{drift_plus_penalty1}.

Define the $T$-round drift as $\Delta_T\triangleq L(T+1)-L(1)= \sum_{n=1}^{N}\frac{1}{2}q_{n,T+1}^2$. Then the $T$-round drift-plus-penalty function can be bounded by:
\begin{align}
\Delta_T+V\sum_{t=1}^{T}  U_t &\leq \sum_{t=1}^{T}\left(\theta_0+\sum_{n=1}^{N} q_{n,t}(\tilde{y}_{n,t}-\delta_{n,t})\right)+V\sum_{t=1}^{T}  U_t  \nonumber\\
&=\theta_0T+ \sum_{t=1}^{T} \left(\sum_{n=1}^{N} q_{n,t}\tilde{y}_{n,t}+V U_t -\sum_{n=1}^{N} q_{n,t}\delta_{n,t}\right)
\end{align}

We use superscript $^*$ to represent the optimal offline solution of $\mathcal{P}4$ ($\sigma_t$ is not an optimization variable), superscript $^\dagger$ to represent the classical drift-plus-penalty algorithm, i.e., $\min_{\left\{\beta_{n,t}\right\}} VU_t + \sum_{n=1}^{N}\beta_{n,t}q_{n,t}E_{n,t}$, and  $^\ddagger$ to represent our proposed estimated-drift-plus-penalty algorithm that solves $\mathcal{P}6$.

The $T$-round drift-plus-penalty is bounded by:
\begin{align} 
\Delta_T^\ddagger+V\sum_{t=1}^{T}  U_t^\ddagger&\leq \theta_0T+ \sum_{t=1}^{T} \left(\sum_{n=1}^{N} q_{n,t}\tilde{y}_{n,t}^\ddagger+V U_t^\ddagger -\sum_{n=1}^{N} q_{n,t}\delta_{n,t}^\ddagger\right)  \nonumber \\
&\overset{(a)}{\leq} \theta_0T+ \sum_{t=1}^{T} \left(\sum_{n=1}^{N} q_{n,t}\tilde{y}_n^\dagger(t)+V U_t^\dagger -\sum_{n=1}^{N} q_{n,t}\delta_{n,t}^\ddagger\right) \nonumber\\
&=\theta_0T+ \sum_{t=1}^{T} \left(\sum_{n=1}^{N} q_{n,t}\left(y_{n,t}^\dagger+\delta_{n,t}^\dagger\right)+V U_t^\dagger -\sum_{n=1}^{N} q_{n,t}\delta_{n,t}^\ddagger\right) \nonumber \\
&=\theta_0T+ \sum_{t=1}^{T} \left(\sum_{n=1}^{N} q_{n,t}y_{n,t}^\dagger+V U_t^\dagger +\sum_{n=1}^{N} q_{n,t}\left(\delta_{n,t}^\dagger-\delta_{n,t}^\ddagger\right)\right) \nonumber \\
&\overset{(b)}{\leq} \theta_0T+ \sum_{t=1}^{T} \left(\sum_{n=1}^{N} q_{n,t}y_{n,t}^*+V U_t^*+2\delta_0\sum_{n=1}^{N} q_{n,t}\right).  \label{bound_T}
\end{align}

Inequality (a) holds because optimally solving $\mathcal{P}6$ yields a minimum value $\sum_{n=1}^{N} q_{n,t}\tilde{y}_{n,t}^\ddagger+V U_t^\ddagger$ for each $t$.
Inequality (b) holds since the drift-plus-penalty algorithm achieves the minimum value of $\sum_{n=1}^{N} q_{n,t}y_{n,t}+V U_t$, and thus plugging in the optimal offline policy on the right-hand-side increases the value.

Now we bound the right-hand-side of \eqref{bound_T}. Note that $q_{n,t+1}-q_{n,t}\leq \theta_n, \forall t,n$, and thus
\begin{align}
& q_{n,t}=q_{n,t}-q_{n,1}=\sum_{\tau=1}^{t-1}(q_{n,\tau+1}-q_{n,\tau})\leq (t-1)\theta_n,\label{q_b1}\\
& q_{n,t}y_{n,t}^*=(q_{n,t}-q_{n,1})y_{n,t}^*\leq (t-1)\theta_n^2. \label{q_b2}
\end{align}
Substituting \eqref{q_b1} and \eqref{q_b2} into \eqref{bound_T} yields
\begin{align}
\Delta_T^\ddagger+V\sum_{t=1}^{T}  U_t^\ddagger
&\leq  \theta_0T+ V \sum_{t=1}^{T} U_t^*+ \sum_{t=1}^{T} \sum_{n=1}^{N}(t-1)\theta_n^2 + 
2\delta_0\sum_{t=1}^{T}\sum_{n=1}^{N} (t-1) \theta_n \nonumber\\
&= \theta_0T+ V\sum_{t=1}^{T} U_t^*+\theta_0T(T-1)+T(T-1)\delta_0\sum_{n=1}^{N} \theta_n \nonumber\\
&=V\sum_{t=1}^{T}  U_t^*+\theta_0T^2+T(T-1)\delta_0\sum_{n=1}^{N} \theta_n. \label{bound_sum}
\end{align}
Notice that $\Delta_T^\ddagger\geq 0$, \eqref{bound_Ut} in Theorem \ref{theorem_algo} can be derived from \eqref{bound_sum} by dividing both sides by $V$. As $U_t>0$, and for $\forall n$, $\frac{1}{2}q_{n,T+1}^2 \leq \Delta_T$, we get
\begin{align}
	\sum_{t=1}^{T}y_{n,t}&=\sum_{t=1}^{T}\beta_{n,t}E_{n,t}-\bar{E}_n\leq \sum_{t=1}^{T} q_{n,t+1}-q_{n,t} =q_{n,T+1} \nonumber\\
	&\leq \sqrt{2\Delta_T} = \sqrt{2V\sum_{t=1}^{T}  U_t^*+2\theta_0T^2+2T(T-1)\delta_0\sum_{n=1}^{N} \theta_n}.
\end{align}
Thus eq. \eqref{bound_energy} in Theorem \ref{theorem_algo} is proved.



\end{document}